\documentclass[11pt]{article}
\usepackage[a4paper, left=3cm, right=3cm, top=3cm, bottom=3cm]{geometry}
\usepackage[utf8]{inputenc}
\usepackage{graphics}
\usepackage{times}
\usepackage{amsmath, amsfonts, amssymb, amsthm}
\usepackage{stmaryrd}
\usepackage{appendix}
\usepackage[hidelinks]{hyperref}

\usepackage{rotating}

\usepackage{todonotes}
\usepackage{caption}
\usepackage{subcaption}
\usepackage{algorithmic}
\usepackage{algorithm}
\usepackage{diagbox}
\usepackage{tablefootnote}
\usepackage{array}
\usepackage{makecell}

\usepackage{booktabs}
\usepackage{enumitem}

\long\def\acks#1{\vskip 0.3in\noindent{\large\bf Acknowledgments}\vskip 0.2in
\noindent #1}

\newtheorem{definition}{Definition}
\newtheorem{assumption}{Assumption}
\newtheorem{proposition}{Proposition}
\newtheorem{theorem}{Theorem}
\newtheorem{lemma}{Lemma}
\newtheorem{corollary}{Corollary}

\newtheoremstyle{myrem}%
{2pt}% Space above
{2pt}% Space below
{}% Body font 
{}% Indent amount
{\bfseries}% ⟨Theorem head font⟩
{.}% ⟨Punctuation after theorem head ⟩
{.5em}% Space after theorem head 
{}%
\theoremstyle{myrem}

\newcommand{\N}{\mathbb{N}}
\newcommand{\R}{\mathbb{R}}
\newcommand{\E}{\mathbb{E}}

\newcommand{\cK}{\mathcal K}
\newcommand{\cO}{\mathcal O}

\newcommand{\cL}{\mathcal{L}}

\newcommand{\Proba}{\mathbb{P}}

\newcommand{\ind}[1]{\mathbf 1_{#1}}

\newcommand{\setint}[1]{\llbracket #1 \rrbracket}

\DeclareMathOperator*{\argmin}{argmin}
\DeclareMathOperator*{\argmax}{argmax}
\DeclareMathOperator{\Tr}{tr}
\DeclareMathOperator{\Var}{Var}

\DeclareMathOperator{\prox}{Prox}

\DeclareMathOperator{\sparse}{sparse}

\newcommand{\softhresh}{\mathcal{T}}

\newcommand{\diag}{\mathrm{diag}}

\newcommand{\sign}{\mathrm{sign}}

\newcommand{\binomial}{\mathrm{Bin}}
\newcommand{\op}{\mathrm{op}}

\newcommand{\cX}{\mathcal X}
\newcommand{\cY}{\mathcal Y}

\newcommand{\wh}{\widehat}

\title{Robust Methods for High-Dimensional Linear Learning}

\author{Ibrahim Merad%
\thanks{LPSM, UMR 8001, Universit\'e Paris Diderot, Paris, France}\\
\and
St\'ephane Ga\"iffas%
\thanks{LPSM, UMR 8001, Universit\'e Paris Diderot, Paris, France and DMA, École normale supérieure}
}
\begin{document}

\maketitle

\begin{abstract}
    We propose statistically robust and computationally efficient linear learning methods in the high-dimensional batch setting, where the number of features $d$ may exceed the sample size $n$.
    We employ, in a generic learning setting, two algorithms depending on whether the considered loss function is gradient-Lipschitz or not.
    Then, we instantiate our framework on several applications including vanilla sparse, group-sparse and low-rank matrix recovery. 
    This leads, for each application, to efficient and robust learning algorithms, that reach near-optimal estimation rates under heavy-tailed distributions and the presence of outliers.
    For vanilla $s$-sparsity, we are able to reach the $s\log (d)/n$ rate under heavy-tails and $\eta$-corruption, at a computational cost comparable to that of non-robust analogs.
    We provide an efficient implementation of our algorithms in an open-source \texttt{Python} library called \texttt{linlearn}, by means of which we carry out numerical experiments which confirm our theoretical findings together with a comparison to other recent approaches proposed in the literature.

\end{abstract}

% \tableofcontents

\section{Introduction}

Learning from heavy tailed or corrupted data is a long pursued challenge in statistics receiving considerable attention in literature~\cite{huber1964robust, hampel1971, huber2004robust, HeavyTails, diakonikolas2019robust, audibert2011robust} and gaining an additional degree of complexity in the high-dimensional setting~\cite{lugosi2019regularization,lecue2020robust, dalalyan2019outlier, balakrishnan2017computationally, liu2019high, liu2020high, fan2021shrinkage}. 
Sparsity inducing penalization techniques~\cite{donoho2000high, buhlmann2011statistics} are the go-to approach for high-dimensional data and have found many applications in modern statistics~\cite{tibshirani1996regression, buhlmann2011statistics, donoho2000high, hastie2015statistical}. 
A clear favourite is the Least Absolute Shrinkage and Selection Operator (LASSO)~\cite{tibshirani1996regression}. Theoretical studies have shown that under the so-called Restricted Eigenvalue (RE) condition, the latter achieves a nearly-optimal estimation rate~\cite{bickel2009simultaneous, buhlmann2011statistics, negahban2012unified}. Further, a rich literature extensively studies the oracle performances of LASSO in various contexts and conditions~\cite{bunea2007sparsity, lounici2008sup, zhang2008sparsity, zhang2009some, zhao2006model, zou2006adaptive, van2008high, lecue2018regularization, bellec2018slope}. 
Other penalization techniques induce different sparsity patterns or lead to different statistical guarantees, such as, to cite but a few, SLOPE~\cite{bogdan2015slope,su2016slope} which is adaptive to the unknown sparsity and leads to the optimal estimation rate, OSCAR~\cite{bondell2008simultaneous} which induces feature grouping or group-$\ell_1$ penalization~\cite{yuan2006model, huang2010variable} which induces block-sparsity.  
Other approaches include for instance Iterative Hard Thresholding (IHT)~\cite{blumensath2009iterative, blumensath2010normalized, jain2014iterative, shen2017tight, jain2016structured} whose properties are studied under the Restricted Isometry Property (RIP).
Another close problem is low-rank matrix recovery, involving the nuclear norm as a low-rank inducing convex penalization~\cite{koltchinskii2011nuclear, candes2009exact, candes2011tight, rohde2011estimation, negahban2011estimation, negahban2012restricted}.

The high-dimensional statistical inference methods cited above are, however, not robust: theoretical guarantees are derived under light-tails (generally sub-Gaussian) and the i.i.d assumption. Unfortunately, these assumptions fail to hold in general, for instance, it is known that financial and biological data often displays heavy-tailed behaviour~\cite{fan2021shrinkage} and outliers or corruption are not uncommon when handling massive amounts of data which are tedious to thoroughly clean. Moreover, the majority of the previous references focus on the oracle performance of estimators as opposed to providing guarantees for explicit algorithms to compute them. 
A natural question is therefore : \emph{can one build alternatives to such high-dimensional estimators that are robust to heavy tails and outliers, that are computationally efficient and achieve rates similar to their non-robust counterparts ?}
Recent advances about robust mean estimation~\cite{catoni2012challenging, lugosi2021robust, diakonikolas2020outlier, Depersin2019RobustSE, lei2020fast} gave a strong impulse in the field of robust learning~\cite{lecue2020robust1, HeavyTails, pmlr-v97-holland19a, diakonikolas2019robust, cherapanamjeri2020optimal, bakshi2021robust}, including the high-dimensional setting~\cite{liu2019high,liu2020high, balakrishnan2017computationally} which led to significant progress towards a positive answer to this question.

However, to the best of our knowledge, the solutions proposed until now are all suboptimal in one way or another. The shortcomings either lie in the obtained statistical rate: which is sometimes significantly far from optimal, or in the robustness: most works consider heavy tailed and corrupted data separately and only very limited amounts of corruption, or in computational complexity: some corruption-filtering algorithms are too heavy and do not scale to real world applications.

In this paper, we propose explicit algorithms to solve multiple sparse estimation problems with high performances in all previous aspects. In particular, our algorithm for vanilla sparse estimation enjoys a nearly optimal statistical rate (up to a logarithmic factor), is simultaneously robust to heavy tails and strong corruption (when a fraction of the data is corrupted) and has a comparable computational complexity to a non robust method.

\subsection{Main contributions.} 

This paper combines non-Euclidean optimization algorithms and robust mean estimators of the gradient into explicit algorithms in order to achieve the following main contributions.

\begin{itemize}
    \item We propose a framework for robust high-dimensional linear learning in the batch setting using two linearly converging stage-wise algorithms for high-dimensional optimization based on Mirror Descent and Dual Averaging. These may be applied for smooth and non-smooth objectives respectively so that most common loss functions are covered. The previous algorithms may be plugged with an appropriate gradient estimator to obtain explicit robust algorithms for solving a variety of problems.
    
    \item The central application of our framework is an algorithm for ``vanilla'' $s$-sparse estimation reaching the nearly optimal $s\log(d)/n$ statistical rate in the batch setting by combining stage-wise Mirror Descent (Section~\ref{sec:md}) with a simple trimmed mean estimator of the gradient. This algorithm is simultaneously robust to heavy-tailed distributions and $\eta$-corruption of the data\footnote{We say that data is $\eta$-corrupted for some number $0<\eta <1/2$ if an $\eta$ fraction of the samples is replaced by arbitrary (and potentially adversarial) outliers after data generation.}. This improves over previous literature which considered the two issues separately or required a very restricted value of the corruption rate $\eta$.

    \item In addition to vanilla sparsity, we instantiate our procedures for group sparse estimation and low-rank matrix recovery, in which different metrics on the parameter space are induced and used to measure the statistical error on the gradient (Section~\ref{sec:applis}). For heavy-tailed data and $\eta$-corruption, the gradient estimator we propose for vanilla sparsity enjoys an optimal statistical rate with respect to the induced metric while the one proposed for group sparsity is nearly optimal up to a logarithmic factor. Moreover, for heavy tailed data and a limited number of outliers\footnote{For low-rank matrix recovery, our estimator is based on Median-Of-Means so that the number of tolerated outliers is up to $K/2$ where $K$ is the number of blocks used for estimation (see Section~\ref{sec:appliLowRank})}, our proposed gradient estimator for low-rank matrix recovery is nearly optimal. Thus, our solutions to each of these problems are the most robust yet in the literature.
    
    \item Our algorithms offer a good compromise between robustness and computational efficiency with the only source of overhead coming from the robust gradient estimation component. In particular, for vanilla sparse estimation, this overhead is minimal so that the asymptotic complexity of our procedure is \emph{equivalent} to that of non-robust counterparts. This is in contrast with previous works requiring costly sub-procedures to filter out corruption.
        
    \item We validate our results through numerical experiments using synthetic data for regression and real data sets for classification (Section~\ref{sec:exp}). Our experiments confirm our mathematical results and compare our algorithms to concurrent baselines from literature.
    
    \item All algorithms introduced in this paper as well as the main baselines from literature we use for comparisons are implemented and easily accessible in a few lines of code through our \texttt{Python} library called \texttt{linlearn}, open-sourced under the BSD-3 License on \texttt{GitHub} and available here\footnote{\url{https://github.com/linlearn/linlearn}}.
\end{itemize}

\subsection{Related Works}

The general problem of robust linear learning was addressed by~\cite{gaiffas2022robust} where the performance of coordinate gradient descent using various estimators was studied and experimentally evaluated. Several other works~\cite{HeavyTails, pensia2020robust,lecue2020robust1,pmlr-v97-holland19a} deal with this problem, however, they do not consider the high-dimensional setting.

The early work of~\cite{agarwal2012stochastic} focuses on vanilla sparse recovery in the stochastic optimization setting and uses a multistage annealed LASSO algorithm where the penalty shrinks progressively. The method reaches the nearly optimal $s\log(d)/n$ rate, however, it is not robust since the data is assumed i.i.d sub-Gaussian. The subsequent work of~\cite{sedghi2014multi} extends this framework to other problems such as additive sparse and low-rank matrix decomposition by changing the optimization algorithm but the sub-Gaussian assumption remains necessary.

More recently,~\cite{Juditsky2020SparseRB} proposed a stochastic optimization mirror descent algorithm which computes multiple solutions on disjoint subsets of the data and aggregates them with a Median-Of-Means type procedure. 
The final solution achieves the rate $s\log(d)/n$ under $s$-vanilla sparsity with sub-Gaussian deviation and an application to low-rank matrix recovery is also developed. 
This aggregation method can handle some but not all heavy-tailed data. 
For instance, if the data follows a Pareto($\alpha$) distribution then the analysis yields a statistical error with a factor of order $d^{1/\alpha}$ which is not acceptable in a high-dimensional setting. 
Moreover, the presence of outliers is not considered so that the given bounds do not measure the impact of corruption. Nevertheless, the combination of the restarted mirror descent optimization procedure proposed in~\cite[Algorithm 1]{Juditsky2020SparseRB} with the proper robust gradient estimators yields a fast and highly robust learning algorithm in the batch setting which we present in Section~\ref{sec:md}. We also exploit the same core ideas in Section~\ref{sec:da} to extend our framework to a wider range of objective functions.

Other works consider high-dimensional linear learning methods that are robust to corrupted data. 
An outlier robust method for mean and covariance estimation in the sparse $\eta$-contaminated\footnote{$\eta$-contamination refers to the case where the data is sampled from a mixed distribution $(1-\eta)P + \eta Q$ where $P$ is the true data distribution and $Q$ is an arbitrary one.} high-dimensional setting is proposed by~\cite{balakrishnan2017computationally} along with theoretical guarantees. 
By extension, these also apply to several problems of interest such as sparse linear estimation or sparse GLMs. The idea is to use an SDP relaxation of sparse PCA~\cite{d2004direct} in order to adapt the filtering approach from~\cite{diakonikolas2019robust}, which relies on the covariance matrix to detect outliers, to the high-dimensional setting. However, the need to solve SDP problems makes the algorithm computationally costly. In addition, the true data distribution is assumed Gaussian and the considered $\eta$-contamination framework is weaker than $\eta$-corruption which allows for adversarial outliers.

The previous ideas were picked up again in the later work of~\cite{liu2020high} who proposes a robust variant of IHT for sparse regression on $\eta$-corrupted data. Unfortunately, these results suffer from several shortcomings since data needs to be Gaussian with a known or sparse covariance matrix. 
Moreover, the gradient estimation subroutine, which is reminiscent of~\cite{balakrishnan2017computationally}, is computationally heavy since it requires solving SDP problems as well and a number of samples scaling as $s^2$ instead of $s$ is required. This seems to come from the fact that sparse gradients need to be estimated, in which case the $s^2$ dependence is unavoidable based on an oracle lower bound for such an estimation~\cite{diakonikolas2017statistical}.

Recently,~\cite{dalalyan2019outlier} derived oracle bounds for a robust estimator in the linear model with Gaussian design and a number of adversarially contaminated labels. 
Although optimal rates in terms of the corruption are achieved, this setting excludes corruption of the covariates and does not apply for heavy-tailed data distributions. \cite{minsker2022robust} similarly consider the linear model and derive oracle bounds for a robustified SLOPE objective which is adaptive to the sparsity level. Remarkably, they achieve the optimal $s\log(ed/s)/n$ dependence in the heavy-tailed corrupted setting. However, corruption is restricted to the labels and the dependence of the result thereon significantly degrades if the covariates or the noise are not sub-Gaussian. In contrast, the very recent work of~\cite{sasai2022robust} considers sparse estimation with heavy-tailed and $\eta$-corrupted data and derives a nearly optimal estimation bound using an algorithm which filters the data before running an $\ell_1$-penalized robust Huber regression which corresponds to a similar approach to~\cite{pensia2020robust} where the non sparse case was treated. Although the $s\log(d)/n$ rate is achieved with optimal robustness, this claim only applies for regression under the linear model with some restrictive assumptions such as zero mean covariates. In addition, little attention is granted to the practical aspect and no experiments are carried out. A later extension~\cite{sasai2022outlier} improves the statistical rate to $s\log(d/s)/n$ for sub-Gaussian covariates and, if the data covariance is known as well, better dependence on the corruption rate is obtained. Robust high-dimensional linear regression algorithms were recently surveyed by~\cite{filzmoser2021robust} with particular attention payed to methods based on dimension reduction, shrinkage and combinations thereof.

Finally,~\cite{liu2019high} proposes an IHT algorithm using robust coordinatewise gradient estimators. 
These results cover the heavy-tailed and $\eta$-corrupted settings separately thanks to Median-Of-Means~\cite{alon1999space, JERRUM1986169, nemirovskij1983problem} and Trimmed mean~\cite{pmlr-v28-chen13h, yin2018byzantine} estimators respectively. 
However, the corruption rate $\eta$ is restricted to be of order at most $O(1/(\sqrt{s}\log(nd)))$ and the question of elaborating an algorithm which is simultaneously robust to both corruption and heavy tails is left open. 

We summarize the settings and results of the previously mentioned works, along with ours on vanilla sparse estimation, in Table~\ref{table:comparison} which focuses on robust papers with explicit algorithms.

\begin{table}[!ht]%[H]
\centering
\begin{tabular}{l|ccccl}
\makecell{Method} & \makecell{Statistical \\rate} & \makecell{Iteration\\ Complexity} & \makecell{Data dist.\\and corruption} & \makecell{Loss} \\
\hline
\makecell{AMMD\\ This paper\\Section~\ref{sec:md}} & \makecell{$O\Big(\sqrt{s}\sqrt{\eta + \frac{\log(d/\delta)}{n}}\Big)$\\with $\Proba \geq 1-\delta$} & \makecell{$O(nd)$} & \makecell{$L_4$ covariates,\\$L_2$ labels,\\$\eta$-corruption.} & \makecell{Lip. smooth,\\ QM.} \\[5pt]
\makecell{AMDA\\ This paper\\Section~\ref{sec:da}} & \makecell{$O\Big(\sqrt{s}\sqrt{\eta + \frac{\log(d/\delta)}{n}}\Big)$\\with $\Proba \geq 1-\delta$} & \makecell{$O(nd)$} & \makecell{$L_4$ covariates,\\$L_2$ labels,\\$\eta$-corruption.} & \makecell{Lipschitz, \\PLM.} \\[5pt]
\makecell{\hyperlink{cite.balakrishnan2017computationally}{(Balakrishnan}\\\hyperlink{cite.balakrishnan2017computationally}{et al., 2017)}} & \makecell{$O\Big(\|\theta^\star\|_2 \big( \eta \log(1/\eta) +$\\$ s\sqrt{\log(d/\delta)/n} \big)\Big)$\\with $\Proba \geq 1-\delta$} & \makecell{$\Omega(nd^2 \!+\! d^3)$} & \makecell{Gaussian,\\ $\eta$-contamination.} & \makecell{LSQ, GLMs,\\Logit.} \\[5pt]
\makecell{\hyperlink{cite.liu2020high}{(Liu et al.,}\\\hyperlink{cite.liu2020high}{2020)}} & \makecell{$O\big(\eta \vee s\sqrt{\log(d/\delta)/n}\big)$\\with $\Proba \geq 1-\delta$} & \makecell{$\Omega(nd^2 \!+\! d^3)$} & \makecell{Gaussian,\\$\eta$-corruption.} & \makecell{LSQ.} \\[5pt]
\makecell{\hyperlink{cite.liu2019high}{(Liu et al.,}\\\hyperlink{cite.liu2019high}{2019)} (MOM)} & \makecell{$O\Big(\sqrt{s\log(d)/n}\Big)$\\with $\Proba \geq 1-d^{-2}$} & \makecell{$O(nd)$} & \makecell{$L_4$ covariates,\\linear/logit model.} & \makecell{LSQ, Logit.} \\[5pt]
\makecell{\hyperlink{cite.liu2019high}{(Liu et al.,}\\\hyperlink{cite.liu2019high}{2019)} (TMean)} & \makecell{$O\Big(\eta \sqrt{s}\log(nd)$\\ $+ \sqrt{s\log(d)/n}\Big)$\\with $\Proba \geq 1-d^{-2}$} & \makecell{$O(nd \log(n))$} & \makecell{sub-Gaussian,\\ $\eta$-corruption.} & \makecell{LSQ, Logit.} \\[5pt]
\makecell{\hyperlink{cite.Juditsky2020SparseRB}{(Juditsky et}\\\hyperlink{cite.Juditsky2020SparseRB}{al., 2020)}} & \makecell{$O\big(\!\sqrt{\!s\log(d)\!\log(1/\delta)/n}\big)$\\with $\Proba \geq 1-\delta$} & \makecell{$O(d)$\\(stochastic\\optim.)} & \makecell{$L_2$ gradient.} & \makecell{Lip. smooth,\\ QM.} \\[5pt]
\makecell{\hyperlink{cite.sasai2022robust}{(Sasai, 2022)}} & \makecell{$O\big(\sqrt{\eta}\!+\!\sqrt{\!s\!\log(d/\delta)/n}\big)$\\with $\Proba \geq 1-\delta$} & \makecell{Preliminary\\$\Omega(nd^2 \!+\! d^3)$\\then $O(nd)$} & \makecell{zero-mean \& $L_8$\\covariates,\\indep. noise,\\$n\gtrsim s^2\log(d/\delta)$.} & \makecell{Penalized\\Huber.} \\[5pt]
% \hline
\end{tabular}
\caption{Summary of the main hypotheses and results of our proposed algorithms and related works in the literature on vanilla sparse estimation. The statistical rate column gives the derived error bound on $\|\wh \theta - \theta^\star\|_2$ between the estimated and true parameter and the associated confidence. In the ``Data distribution and corruption'' column, rows with no reference to corruption correspond to methods which do not consider it. In the ``Loss'' column, the following abbreviations are used : QM = quadratic minorization (Assumption~\ref{asm:quadgrowth}), PLM = pseudo-linear minorization (Assumption~\ref{asm:pseudolingrowth}), LSQ = least squares, GLM = generalized linear model, Logit = Logistic, Lip. smooth = Lipschitz smooth (gradient Lipschitz).}
\label{table:comparison}
\end{table}

\subsection{Agenda}

The remainder of this document is structured as follows : Section~\ref{sec:setting} lays out the setting including the definition of the objective and our assumptions on the data.
Sections~\ref{sec:md} and~\ref{sec:da} define optimization algorithms based on Mirror Descent and Dual Averaging addressing the cases of smooth and non-smooth losses respectively. 
Both Sections state convergence results for their respective algorithms. 
Section~\ref{sec:applis} considers instantiations of our general setting to vanilla sparse, group sparse and low-rank matrix estimation for a general loss. 
In each case, the norm $\|\cdot\|$ and dual norm $\|\cdot\|_{*}$ are instantiated and a robust and efficient gradient estimator is proposed so that, combined with the results of Sections~\ref{sec:md} and~\ref{sec:da}, we obtain solutions with nearly optimal statistical rates (up to logarithmic terms) in each case. 
Finally, Section~\ref{sec:exp} presents numerical experiments on synthetic and real data sets which demonstrate the performance of our proposed methods and compare them with baselines from recent literature.

\section{Setting, Notation and Assumptions}\label{sec:setting}

We consider supervised learning from a data set $(X_i, Y_i)_{i=1}^n $ from which the majority is distributed as a random variable $(X, Y) \in \cX \times \cY$ where the covariate space $\cX$ is a high-dimensional Euclidean space and the label space $\cY$ is $\R$ or a finite set. 
The remaining minority of the data are called outliers and may be completely arbitrary or even adversarial. 
Given a loss function $\ell : \wh \cY \times \cY \to \R$  where $\wh \cY$ is the prediction space, our goal is to optimize the unobserved objective
\begin{equation}\label{eq:objective}
    \cL (\theta) = \E [\ell(\langle \theta, X \rangle , Y)]
\end{equation}
over a convex set of parameters $\Theta$, 
where the expectation is taken w.r.t. the joint distribution of $(X, Y)$.
Given an optimum $\theta^\star = \argmin_{\theta \in \Theta} \cL(\theta)$ (assumed unique), we are also interested in controlling the estimation error $\|\theta - \theta^\star\|$ where $\|\cdot\|$ is a norm on $\Theta$, which will be defined according to each specific problem (see Section~\ref{sec:applis} below). 
Moreover, we assume that the optimal parameter is sparse according to an abstract sparsity measure $S : \Theta \to \N$.
\begin{assumption}\label{asm:sparse}
The optimal solution $\theta^\star$ is $s$-sparse for some integer $s$ smaller than the problem dimension i.e. $S(\theta^\star) \leq s$. Additionally, for any $s$-sparse vector $\theta \in \Theta$ we have the inequality $\|\theta\| \leq \sqrt{s}\|\theta\|_2$ and an upper bound $\bar{s} \geq s$ on the sparsity is known.
\end{assumption}
The precise notion of sparsity will be determined later in Section~\ref{sec:applis} through the sparsity measure $m$ depending on the application at hand. The simplest case corresponds to the conventional notion of vector sparsity where $\cX = \R^d$ for some large $d$, $\Theta \subset \R^d$ and the sparsity measure $S(\theta) = \sum_{j\in \setint{d}}\ind{\theta_j \neq 0}$ counts the number of non-zero coordinates. However, as we intend to also cover other forms of sparsity later, we do not fix this setting right away. Note that the required knowledge of $\bar{s} \geq s$ in Assumption~\ref{asm:sparse} is common in the thresholding based sparse learning literature~\cite{blumensath2009iterative, blumensath2010normalized, jain2014iterative, jain2016structured, Juditsky2020SparseRB}. Adaptive methods to unknown sparsity exist~\cite{bogdan2015slope, bellec2018slope, su2016slope} although designing a robust version thereof is beyond the scope of this work.

Since the objective~\eqref{eq:objective} is not observed due to the distribution of the data being unknown, statistical approximation will be necessary in order to recover an approximation of $\theta^\star$. Instead of estimating the objective itself (which is of limited use for optimization), we will rather compute estimates of the gradient
\begin{equation}
    g(\theta) := \nabla_{\theta} \cL(\theta)
\end{equation}
in order to run gradient based optimization procedures. 
Note that, since we consider the high-dimensional setting, a standard gradient descent approach is excluded since it would incur an error strongly depending on the problem dimension. In order to avoid this, one must use a non Euclidean optimization method as is customary for high-dimensional problems~\cite{Juditsky2020SparseRB, agarwal2012stochastic}.

As commonly stated in the robust statistics literature \cite{catoni2012challenging, lugosi2021robust, lugosi2019mean}, estimating an expectation using a conventional empirical mean only yields values with far from optimal deviation properties in the general case. Several estimators have been proposed which enjoy sub-Gaussian deviations from the true mean and robustness to corruption. Notable examples in the univariate case are the median-of-means (MOM) estimator~\cite{alon1999space, JERRUM1986169, nemirovskij1983problem}, Catoni's estimator~\cite{catoni2012challenging} and the trimmed mean~\cite{lugosi2021robust}. However, in the multivariate case (estimating the mean of a random vector), the optimal sub-Gaussian estimation rate cannot be obtained by a straightforward extension of the previous methods and a line of works~\cite{lugosi2019sub, Hopkins2018MeanEW, pmlr-v99-cherapanamjeri19b, Depersin2019RobustSE, lugosi2021robust, lei2020fast} has pursued elaborating efficient algorithms to achieve it. Most recently, \cite{diakonikolas2020outlier} managed to show that stability based estimators enjoy sub-Gaussian deviations while being robust to corruption of a fraction of the data. However, it is important to remember that all the works we just mentioned measure the estimation error using the Euclidean norm while many other choices are possible which may require the estimation algorithm to be adapted in order to achieve optimal deviations with respect to the chosen norm. This aspect was studied in~\cite{lugosi2019near} who gave a norm-dependent formula for the optimal deviation and an algorithm to achieve it, although the latter has exponential complexity and does not consider the presence of outliers.

This is an important aspect to keep in mind in our high-dimensional setting since we will be measuring the statistical error on the gradient using the dual norm $\|\cdot\|_{*}$ of $\|\cdot\|$ which will never be the Euclidean one.
\begin{equation}
    \|v\|_{*} = \sup_{\|x\| \leq 1} \langle v, x \rangle
\end{equation}

Of course, apart from the way it is measured, the quality of the estimations one can obtain also crucially depends on the assumptions made on the data. We formally state ours here. We denote $|A|$ as the cardinality of a finite set $A$ and use the notation $\setint k = \{ 1, \ldots, k\}$ for any integer $k \in \N \setminus \{ 0 \}$. 

\begin{assumption}
\label{asm:data}
The indices of the training samples $\setint n$ can be divided into two disjoint subsets $\setint n = \mathcal{I} \cup \mathcal{O}$ of \emph{outliers} $\mathcal O$ and \emph{inliers} $\mathcal I$ for which  
we assume the following\textup:~$(a)$ we have $|\mathcal I| > |\mathcal O|;$~$(b)$ the pairs $(X_i, Y_i)_{i \in \mathcal I}$ are i.i.d with distribution $P$ and the outliers $(X_i, Y_i)_{i \in \mathcal O}$ are arbitrary$;$~$(c)$ the distribution $P$ is such that :
\begin{equation}
    \label{eq:XY-moments}
    \E\big[ \|X\|_2^4 \big] < +\infty, \quad \E\big[\|Y X\|_2^2\big] < +\infty \quad
    \text{and} \quad \E\big[|Y|^2\big] < +\infty.
\end{equation}
Moreover\textup, the loss function $\ell$ admits constants $C_{\ell, 1}, C_{\ell, 2}, C'_{\ell, 1}, C'_{\ell, 2}> 0$ such that for all $z, y \in \widehat{\cY}\times \cY$ \textup:
\begin{equation*}
    |\ell(z, y)| \leq C_{\ell, 1} + C_{\ell, 2}|z-y|^2 \quad \text{and} \quad |\ell'(z, y)| \leq C_{\ell, 1}' + C_{\ell, 2}'|z-y|,
\end{equation*}
where $\ell'$ is the derivative of $\ell$ in its first argument.
\end{assumption}
The above hypotheses are sufficient so that the objective function and its gradient exist for any parameter $\theta$ and the gradient admits a second moment. The distribution $P$ is allowed to be heavy-tailed and the conditions on $\ell$ do not go far beyond limiting it to a quadratic behavior and are satisfied by common loss functions for regression and classification\footnote{Including the square loss, the absolute loss, Huber's loss, the logistic loss and the Hinge loss.}. Note that depending on the loss function used and the moment requirements of gradient estimation, the previous moments assumption can be weakened as in~\cite[Assumption 2]{gaiffas2022robust} for instance. However, we stick to this version in this work for simplicity. Depending on the gradient estimator, the number of outliers $|\cO|$ will be bounded in the subsequent statements either by a constant fraction $\eta n$ ($\eta$-corruption) of the sample for some $0 < \eta <1/2$, or by a constant number.

In the two following sections, we will assume that we have a gradient estimator $\widehat{g}$ at our disposal such that for all $\theta \in \Theta$ we have $\widehat g(\theta) = g(\theta) + \epsilon(\theta)$ where $\epsilon(\theta)$ represents the random estimation error of the gradient at $\theta$. Same as for the norm $\|\cdot\|$, this estimator will be precisely defined for each individual application in Section~\ref{sec:applis} in such a way that the error $\|\widehat{g}(\theta) - g(\theta)\|_* = \|\epsilon(\theta)\|_*$ is (nearly) optimally controlled.

In the sequel, we interchangeably use the terms \emph{statistical rate}, \emph{estimation rate}, \emph{deviation rate} or simply \emph{rate} to designate the statistical dependence of the bounds we obtain on the excess risk $\cL(\widehat{\theta}) - \cL(\theta^\star)$ and the parameter error $\|\widehat{\theta} - \theta^\star\|_2$ for a general estimator $\widehat{\theta}.$ This may lead to some confusion since the excess risk is only comparable to the square error $\|\widehat{\theta} - \theta^\star\|_2^2$ up to a constant factor. However, the reader should be able to distinguish the two situations based on context. Note that the previous terms should not be confused with \emph{optimization rate} and \emph{corruption rate}.

\section{The Smooth Case with Mirror Descent}\label{sec:md}

In this section we will assume the loss $\ell$ is smooth, formally :

\begin{assumption}
\label{asm:lipsmoothloss}
For any $y \in \mathcal Y,$ the loss $z \mapsto \ell(z, y)$ is convex, differentiable and $\gamma$-smooth meaning that 
\begin{equation}\label{eq:lipsmooth}
    |\ell'(z, y) - \ell'(z', y)| \leq \gamma |z - z'|
\end{equation} for some $\gamma > 0$ and all $z, z' \in \wh \cY$, where the derivative is taken w.r.t. the first argument.
\end{assumption}
The above assumption is stated in all generality for $z, z'$ belonging to the prediction space $\wh \cY$. For regression or binary classification tasks, we will have $\wh \cY = \R$ and $\ell '(\cdot, y) \in \R$ so that the absolute values are enough to interpret the required inequality. Nonetheless, it can also be extended for $K$-way multiclass classification where $\wh \cY = \R^K$ and $\ell '(\cdot, y) \in \R^K$, in which case the absolute values on both sides of the above inequality should be interpreted as Euclidean norms.

We also make the following quadratic growth assumption~\cite[Definition 4]{necoara2019linear}.

\begin{assumption}\label{asm:quadgrowth}
Let $\theta^\star \in \Theta$ be the optimum of the objective $\cL$ and $\|\cdot\|_2$ the usual Euclidean norm. There exists a constant $\kappa > 0$ such that for all $\theta \in \Theta:$
\begin{equation}\label{eq:quadminoration}
    \cL(\theta) - \cL(\theta^\star) \geq \kappa \|\theta - \theta^\star\|_2^2.
\end{equation}
\end{assumption}
Assumption~\ref{asm:quadgrowth} is similar to but weaker than strong convexity because it only requires the quadratic minorization to hold around the optimum $\theta^\star$ whereas a strongly convex function is minorized by a quadratic function at every point. In the linear regression setting with samples $(X,Y) \in \R^d\times \R$, it is easy to see that the above condition holds as soon as the data follows a distribution with non-singular covariance $\Sigma = \E XX^\top.$ A more general setting where this condition holds is also given in~\cite[Section 3.1]{Juditsky2020SparseRB}. In comparison, the commonly used restricted eigenvalue or compatibility conditions~\cite{bickel2009simultaneous, castillo2015bayesian} roughly require the empirical covariance $\widehat{\Sigma} = \frac{1}{n}\sum_{i\leq n} X_i X_i^\top$ to satisfy $\|\widehat{\Sigma} v\|_2 \geq \kappa \|v\|_2$ for all approximately $s$-sparse vectors $v\in \R^d.$ This was shown to hold for covariates following some well known distributions (e.g. Gaussian with non singular covariance) with a sufficient sample count $n$~\cite{raskutti2010restricted}. However, this is clearly more constraining than Assumption~\ref{asm:quadgrowth}. Some variants of the compatibility condition are formulated in terms of the population covariance $\Sigma$~\cite{van2009conditions, buhlmann2011statistics} but these serve as a basis to show oracle inequalities for LASSO rather than the study of an implementable algorithm.

As a consequence of Assumptions~\ref{asm:data} and~\ref{asm:lipsmoothloss}, the objective gradient is $L$-Lipschitz continuous for some constant $L > 0$, meaning that we have :
\begin{equation}
    \|g(\theta) - g(\theta')\|_* \leq L\|\theta - \theta'\| \quad \forall \theta, \theta' \in \Theta.
\end{equation}
This property is necessary to establish the convergence of the Mirror Descent algorithm~\cite{nemirovskij1983problem} proposed in this section.
Since we adopt a multistage mirror descent procedure as done in~\cite{Juditsky2020SparseRB}, our framework is also similar to theirs. 

\begin{definition}\label{def:dgf}
    A function $\omega : \Theta \to \R$ is a distance generating function if it is a real convex function over $\Theta$ which satisfies :
    \begin{enumerate}
        \item $\omega$ is continuously differentiable and strongly convex w.r.t. the norm $\|\cdot\|$ i.e.
    \begin{equation*}
        \langle \nabla \omega (\theta) - \nabla \omega (\theta'), \theta - \theta' \rangle \geq \|\theta - \theta'\|^2.
    \end{equation*}
    \item We have $\omega(\theta) \geq \omega(0) = 0$ for all $\theta \in\Theta$. 
    \item There exists a constant $\nu > 0$ called the quadratic growth constant such that we have \textup:
    \begin{equation}\label{eq:dgf_quad_growth}
        \omega(\theta) \leq \nu \|\theta\|^2 \quad \forall \theta \in \Theta.
    \end{equation}
    \end{enumerate} 
\end{definition}
We shall see, for individual applications, that one needs to choose $\omega$ in such a way that it is strongly convex and the constant $\nu$ has only a light dependence on the dimension. For a reference point $\theta_0 \in \Theta$, we define $\omega_{\theta_0}(\theta) = \omega(\theta - \theta_0)$ and the associated Bregman divergence :
\begin{equation*}
    V_{\theta_0}(\theta, \theta') = \omega_{\theta_0}(\theta) - \omega_{\theta_0}(\theta') - \langle \nabla \omega_{\theta_0}(\theta'), \theta - \theta' \rangle.
\end{equation*}
Given a step size $\beta > 0$ and a dual vector $u \in \Theta^*$, we define the following proximal mapping :
\begin{align*}
    \prox_{\beta}(u, \theta ; \theta_0, \Theta) :=& \argmin_{\theta' \in \Theta} \{ \langle \beta u, \theta' \rangle + V_{\theta_0}(\theta', \theta)\} \\
    =& \argmin_{\theta' \in \Theta} \{ \langle \beta u - \nabla \omega_{\theta_0} (\theta), \theta' \rangle + \omega_{\theta_0}(\theta')\}.
\end{align*}
The previous operator yields the next iterate of Mirror Descent for previous iterate $\theta$, gradient $u$ and step size $\beta$ with Bregman divergence defined according to the reference point $\theta_0$. Ideally, we would plug $g(\theta)$ as gradient $u$ but since the true gradient is not observed, we replace it with the estimator $\widehat{g}(\cdot)$. All in all, given an initial parameter $\theta_0$ we obtain the following iteration for Mirror Descent :
\begin{equation}\label{eq:iterMDuncorrected}
    \theta_{t+1} = \prox_{\beta}(\widehat{g}(\theta_t), \theta_t; \theta_0, \Theta),
\end{equation} 
with a step size $\beta$ to be defined later according to problem parameters. The previous proximal operator can be computed in closed form in each of the applications we consider in Section~\ref{sec:applis}, see Appendix~\ref{apd:closed-form-prox} for details. We state the convergence properties of the above iteration in the following proposition.
\begin{proposition}\label{prop:uncorrectedMD}
Grant Assumptions~\ref{asm:data} and \ref{asm:lipsmoothloss} so that the objective $\cL$ is $L$-Lipschitz-smooth for some $L>0$. Let mirror descent be run with constant step size $\beta \leq 1/L$ starting from $\theta_0 \in \Theta$ with $\Theta = B_{\|\cdot \|}(\theta_0, R)$ for some radius $R>0$. Let $\theta_1, \dots, \theta_T$ denote the resulting iterates and $\widehat{\theta}_T = \sum_{t=1}^T \theta_t/T,$ then the following inequality holds \textup:
\begin{align*}
    \cL(\widehat{\theta}_T) - \cL(\theta^\star) &\leq \frac{1}{T}\Big(\frac{1}{\beta }V_{\theta_0}(\theta^\star, \theta_0) + \sum_{t=0}^{T-1} \langle \epsilon_t, \theta^\star - \theta_{t+1} \rangle \Big) \\
    &\leq \frac{\nu R^2}{\beta T} + 2\bar{\epsilon}R
\end{align*}
where $\epsilon_t = \widehat{g}(\theta_t) - g(\theta_t)$ and $\bar{\epsilon} = \max_{t=0\dots T-1}\|\epsilon_t\|_{*}.$
\end{proposition}
Proposition~\ref{prop:uncorrectedMD} is proven in Appendix~\ref{proof:uncorrectedMD} based on~\cite[Proposition 2.1]{Juditsky2020SparseRB} and quantifies the progress of mirror descent on the objective value while measuring the impact of the gradient errors. 
The original version in~\cite{Juditsky2020SparseRB} considers a stochastic optimization problem in which a new sample arrives at each iteration providing an unbiased estimate of the gradient so that it is possible to obtain a bound with optimal quadratic dependence on the statistical error. Though the above result is suboptimal in this respect, we will show in the sequel that an optimal statistical rate can still be achieved using a multistage procedure.

Notice that the previous statement only provides guarantees for the average $\widehat{\theta}_T = \sum_{t=1}^T \theta_t/T$. While this is commonplace for online settings, we intuitively expect the last iterate $\theta_T$ to be the best estimate of $\theta^\star$ in our batch setting where all the data is available from the beginning.

In order to address this issue, we define a \emph{corrected} proximal operator given an upper bound $\bar{\epsilon}$ on the statistical error : 
\begin{equation}\label{eq:iterMDcorrected}
    \widehat{\prox}_{\beta}(u, \theta; \theta_0, \Theta) = \argmin_{\theta'\in\Theta} \{\langle \beta u, \theta' \rangle + \beta\bar{\epsilon}\|\theta' - \theta\| + V_{\theta_0}(\theta', \theta)\}.
\end{equation}
For this new operator, the following statement applies.
\begin{proposition}\label{prop:correctedMD}
In the setting of Proposition~\ref{prop:uncorrectedMD}, let mirror descent be similarly run with constant step size $\beta \leq 1/L$ starting from $\theta_0 \in \Theta$ with $\Theta = B_{\|\cdot \|}(\theta_0, R)$ for some radius $R>0$. Let $\theta_1, \dots, \theta_T$ denote the resulting iterates obtained through $\theta_{t+1} = \widehat{\prox}_{\beta}(\widehat{g}(\theta_t), \theta_t;\theta_0, \Theta)$ then the following inequality holds \textup:
\begin{align*}
    \cL(\theta_T) - \cL(\theta^\star) &\leq \frac{1}{T}\Big(\frac{1}{\beta }V_{\theta_0}(\theta^\star, \theta_0) + 2\sum_{t=0}^{T-1} \langle \epsilon_t, \theta^\star - \theta_{t+1} \rangle \Big) \\
    &\leq \frac{\nu R^2}{\beta T} + 4\bar{\epsilon}R,
\end{align*}
where $\bar{\epsilon} = \max_{t=0\dots T-1}\|\epsilon_t\|_{*}$.
\end{proposition}
The proof of Proposition~\ref{prop:correctedMD} is given in Appendix~\ref{proof:correctedMD} and mainly differs from that of Proposition~\ref{prop:uncorrectedMD} in that the introduced correction  allows to show a monotonous decrease of the objective i.e. $\cL(\theta_{t+1}) \leq \cL(\theta_{t})$ letting us draw the conclusion on the last iterate. Nevertheless, we suspect that the correction is not really needed for this bound to hold on $\theta_T$ and consider it rather as an artifact of our proof.

Propositions~\ref{prop:uncorrectedMD} and~\ref{prop:correctedMD} only state a linear dependence of the final excess risk on the statistical error $\bar{\epsilon}$ which leads to a suboptimal statistical rate of $1/\sqrt{n}$. However, the optimal rate of $1/n$ can be achieved by leveraging the sparsity condition on $\theta^\star$ (Assumption~\ref{asm:sparse}) and the quadratic growth condition (Assumption~\ref{asm:quadgrowth}) upon running \emph{multiple stages} of Mirror Descent~\cite{juditsky2011first, Juditsky2020SparseRB}. The idea is that by factoring these two assumptions in, for $T$ big enough in Proposition~\ref{prop:correctedMD} and given $\Bar{s}\geq s,$ it can be shown that the closest $\bar{s}$-sparse element $\sparse_{\Bar{s}}(\theta_T)$ to the last iterate $\theta_T$ is such that $\Theta' = B_{\|\cdot \|}(\sparse_{\Bar{s}}(\theta_T), R')\ni \theta^\star$ with $R' < R.$ Therefore, $\sparse_{\Bar{s}}(\theta_T)$ can serve as the starting point of a new stage of mirror descent on the smaller domain $\Theta'.$ By repeating this trick multiple times, we obtain the following multistage mirror descent algorithm.

\paragraph{Algorithm : Approximate Multistage Mirror Descent (AMMD)}

\begin{itemize}
    \item \textit{Initialization}: Initial parameter $\theta^{(0)}$ and $R>0$ such that $\theta^{\star} \in \Theta := B_{\|\cdot\|}(\theta_0, R)$.
    
    Number of stages $K > 0$. Step size $\beta \leq 1/L$. Quadratic minorization constant $\kappa$.
    
    High probability upperbound $\bar{\epsilon}$ on the error $\|\widehat{g}(\theta) - g(\theta)\|_*$.
    
    Upperbound $\bar{s}$ on the sparsity $s$.
    
    \item Set $R_0 = R$.

    \item Loop over stages $k = 1\dots K$ :
    \begin{itemize}
        \item Set $\theta^{(k)}_0 = \theta^{(k-1)}$ and $\Theta_k = B_{\|\cdot\|}(\theta^{(k)}_0, R_{k-1}).$
        \item Run iteration \begin{equation*}
            \theta^{(k)}_{t+1} = \widehat{\prox}_{\beta}\big(\widehat{g}(\theta^{(k)}_t), \theta^{(k)}_t; \theta^{(k)}_0, \Theta_k \big),
        \end{equation*} for $T_k$ steps with $T_k = \Big\lceil \frac{\nu R_{k-1} }{\beta \bar{\epsilon}}\Big\rceil$.
        \item Set $\theta^{(k)} = \sparse_{\bar{s}} (\widetilde{\theta}^{(k)})$ where $\widetilde{\theta}^{(k)} = \theta^{(k)}_{T_k}$.
        \item Set $R_k = \frac{1}{2}(R_{k-1} + \frac{40\bar{s}\bar{\epsilon}}{\kappa})$.

    \end{itemize}
    \item \textit{Output}: The final stage estimate $\theta^{(K)}$.
\end{itemize}
The AMMD algorithm borrows ideas from~\cite{juditsky2011first, juditsky2014deterministic, Juditsky2020SparseRB} aiming to achieve linear convergence using mirror descent. The main trick lies in the fact that performing multiple stages of mirror descent allows to repeatedly restrict the parameter space into a ball of radius $R_k$ which shrinks geometrically with each stage. 
In this work, we find that the radius $R_k$ evolves following a special contraction as a result of the statistical error being factored in. Note that, although a few instructions of AMMD are stated in terms of unknown quantities, the procedure may be simplified to get around this difficulty with satisfactory results, see Section~\ref{sec:exp} for details. We show that the above procedure allows to improve the result of Proposition~\ref{prop:correctedMD} to achieve a fast statistical rate. 
The following statement expresses the theoretical properties of AMMD.
\begin{theorem}\label{thm:MDoptimalrate}
    Grant Assumptions~\ref{asm:sparse},~\ref{asm:data},~\ref{asm:lipsmoothloss} and~\ref{asm:quadgrowth}. Let $L>0$ denote the Lipschitz smoothness constant for the objective $\cL$. Assume approximate Mirror Descent is run with step size $\beta \leq 1/L$ starting from $\theta_0 \in\Theta$ such that $\theta^\star \in B_{\|\cdot\|}(\theta_0, R)$ for some $R > 0$ and using a gradient estimator $\widehat{g}$ with error upperbound $\bar{\epsilon}$ as in Proposition~\ref{eq:iterMDuncorrected}, then after $K$ stages we have the inequalities \textup:
    \begin{equation*}
        \big\|\theta^{(K)} - \theta^\star\big\| \leq \sqrt{2\bar{s}} \big\|\theta^{(K)} - \theta^\star\big\|_2 \leq 2\sqrt{2\bar{s}} \big\|\widetilde{\theta}^{(K)} - \theta^\star\big\|_2 \leq   2^{-(K-1)/2}R + \frac{40\bar{s}\bar{\epsilon}}{\kappa},
    \end{equation*}
    \begin{equation*}
        \cL(\widetilde{\theta}^{(K)}) - \cL(\theta^\star) \leq 10\bar{\epsilon} \Big(2^{-K}R + \frac{40 \bar{s}\bar{\epsilon}}{\kappa}\Big).
    \end{equation*}
    Moreover\textup, the corresponding number of necessary iterations is bounded by \textup:
    \begin{equation*}
        T = \sum_{k=1}^K T_k \leq \frac{2R \nu}{\beta \bar{\epsilon}} + K\Big(1 + \frac{40 \nu \bar{s}}{\kappa \beta}\Big).
    \end{equation*}
\end{theorem}
Theorem~\ref{thm:MDoptimalrate} is proven in Appendix~\ref{proof:thmMD} and may be compared to~\cite[Theorem 2.1]{Juditsky2020SparseRB}. Both statements bound the risk in terms of the objective and parameter error by the sum of an exponentially vanishing optimization error and a statistical error term. Note that the exponential optimization rate in the number of stages $K$ also holds in the number of iterations since successive stages contain a geometrically decreasing number of them. Theorem~\ref{thm:MDoptimalrate} expresses the statistical error in terms of a bound $\bar{\epsilon}$ and will thus lead to a high confidence statement when combined with a bound on $\bar{\epsilon}$ (see Section~\ref{sec:applis}). In contrast, Theorem 2.1 of~\cite{Juditsky2020SparseRB} is a result in expectation which is later used to derive a high confidence bound on an aggregated estimate. 

One can see that Theorem~\ref{thm:MDoptimalrate} exhibits a dependence in $\bar{\epsilon}^2$ of the excess risk upperbound so that the suboptimal statistical rate in Propositions~\ref{prop:uncorrectedMD} and~\ref{prop:correctedMD} is improved into a fast rate as announced. This is accomplished by shrinking the size of the considered parameter set through the stages until it reaches the scale of the statistical error, yielding an optimal rate. This shrinkage is achieved thanks to the choice of stage-length $T_k = \Omega (R_{k-1})$ in AMMD leading to a bound in terms of $R$ rather than $R^2$ in Proposition~\ref{prop:correctedMD}. Combined with Assumption~\ref{asm:quadgrowth}, this implies that a square root function is applied to $R_k$ after each stage, see the proof for further details. We can now turn to the case of a non smooth loss function $\ell$.

\section{The Non Smooth Case with Dual Averaging}\label{sec:da}

In the previous section, we saw how sparse estimation can be performed using the Mirror Descent algorithm to optimize a smooth objective with a non Euclidean metric on the parameter space. The smoothness property is necessary for these results to hold so that many loss functions not satisfying it are left uncovered. Therefore, we propose to use another algorithm for non smooth objectives. The alternative is the Dual Averaging algorithm~\cite{nesterov2009primal} which was already used for non smooth sparse estimation in~\cite{agarwal2012stochastic} for instance. Since the original algorithm requires to average the iterates to obtain a parameter with provable convergence properties, we instead use a variant~\cite{nesterov2015quasi} for which such properties apply for individual iterates.

The smoothness condition in Assumption~\ref{asm:lipsmoothloss} is no longer required but we still need to replace it with a Lipschitz property :
\begin{assumption}\label{asm:lipschitz}
There exists a positive constant $M > 0$ such that the objective $\cL$ is $M$-Lipschitz w.r.t. the norm $\|\cdot\|$ i.e. for all $\theta, \theta' \in \Theta$ it holds that \textup:
\begin{equation*}
    \cL(\theta) - \cL(\theta') \leq M \|\theta - \theta'\|.
\end{equation*}
\end{assumption}

We also replace Assumption~\ref{asm:quadgrowth} by the following weaker assumption :
\begin{assumption}\label{asm:pseudolingrowth}
There exist positive constants $\kappa, \lambda > 0$ such that the following inequality holds \textup: 
\begin{equation*}
    \cL (\theta) - \cL(\theta^\star) \geq \frac{\kappa\|\theta - \theta^\star\|_2^2}{\lambda + \|\theta - \theta^\star\|_2}.
\end{equation*}
\end{assumption}
We introduce this (to our knowledge) previously unknown assumption in the literature which we call the \emph{pseudo-linear} growth assumption in order to better suit the setting of this section. Indeed, few non-smooth loss functions, if any, result in quadratically growing objectives as Assumption~\ref{asm:quadgrowth} requires. Note that the lower bound of Assumption~\ref{asm:pseudolingrowth} is linear away from the optimum i.e. for big $\|\theta - \theta^\star\|_2$  and behaves quadratically around it. This assumption is also weaker than a linear lower bound proportional to $\|\theta - \theta^\star\|_2$ because of its quadratic behaviour around the optimum. We will show that, for $\kappa$ big enough, this minorization suffices to obtain linear convergence to a solution with fast statistical rate.

Analogously to Mirror Descent's distance generating function $\omega$, we let $\omega : \Theta \to \R^+$ be the \emph{prox-function}. We choose to denote it similarly since it plays an analogous role for Dual Averaging and has the same properties as those listed in Definition~\ref{def:dgf}.

Let $(a_t)_{t\geq0}$ be a sequence of step sizes and $(\gamma_t)_{t \geq 0}$ a non decreasing sequence  of positive scaling coefficients. The DA procedure is defined, given an initial $\theta_0 \in \Theta$, by the following scheme :
\begin{align*}
    s_t = \frac{1}{A_t} \sum_{i=0}^t a_i \widehat{g}_i \quad \text{with} \quad &\widehat{g}_i = \widehat{g}(\theta_i) \:\: \text{and}\:\: g_i = g(\theta_i) \quad \forall i =0, \dots , T. \\
    A_t = \sum_{i=0}^t a_i \quad \text{and}\quad &\theta_t^+ = \argmin_{\theta\in \Theta}A_t \langle s_t, \theta \rangle + \gamma_t \omega(\theta). \\
    \theta_{t+1} = (1 - \tau_t)\theta_t + &\tau_t \theta_t^+ \quad\quad \text{where} \quad \quad\tau_t = \frac{a_{t+1}}{A_{t+1}}. 
\end{align*}

\begin{proposition}\label{prop:DA}
    Grant Assumption~\ref{asm:lipschitz}, let Dual Averaging be run following the above scheme\textup, let $R>0$ such that $\Theta \subseteq B_{\|\cdot\|}(\theta_0, R)$ and denote $\bar{\epsilon} = \max_i \|\epsilon_i\|_*,$ we have the following inequality \textup:
    \begin{equation*}
        A_t(\cL(\theta_t) - \cL(\theta^\star)) + \frac{\gamma_t}{2}\|\theta_t^+ - \theta^\star\|^2 \leq \gamma_t \omega(\theta^\star) + \sum_{i=0}^t\frac{a_i^2}{2\gamma_{i-1}}\|g_i\|_*^2 + 4A_t R \bar{\epsilon}.    
    \end{equation*}
    In particular\textup, by choosing $a_i = 1$ and $\gamma_i = \sqrt{i+1}$ for all $i$ we get \textup:
    \begin{equation*}
        \cL(\theta_t) - \cL(\theta^\star) \leq \frac{1}{\sqrt{t}}\big( \omega(\theta^\star) + M^2\big) + 4 R \bar{\epsilon}.
    \end{equation*}
\end{proposition}
The proof of Proposition~\ref{prop:DA} is given in Appendix~\ref{proof:DA} and is inspired from~\cite[Theorem 3.1]{nesterov2015quasi}. In this result, we manage to obtain a statement in terms of the individual iterates $\theta_t$ thanks to the running average performed in the above scheme whereas the initial study of dual averaging focused on the \emph{average} of the iterates~\cite{nesterov2009primal}. Notice that, due to Assumption~\ref{asm:lipsmoothloss} being dropped, the convergence speed degrades to $1/\sqrt{t}$ as opposed to $1/t$ previously. This convergence speed is the fastest possible and cannot be improved with a different choice of $a_i$ and $\gamma_i$. Most importantly, Proposition~\ref{prop:DA} quantifies the impact of the errors on the gradients on the quality of the optimisation result and shows that it remains controlled in this case too.

As in the previous section, the statistical rate we initially obtain is suboptimal and a multistage procedure is needed to improve it. The idea is the same as in Section~\ref{sec:md} and consists in sparsifying the final iterate in Proposition~\ref{prop:DA} and using it as the initial point of a new optimization stage which takes place on a narrower domain. We make the resulting algorithm explicit below.

\paragraph{Algorithm : Approximate Multistage Dual Averaging (AMDA)}
\begin{itemize}
    \item \textit{Initialization}: Initial parameter $\theta_0$ and $R>0$ such that $\theta^{\star} \in \Theta := B_{\|\cdot\|}(\theta_0, R)$.
    
    Pseudolinear minorization constants $\kappa, \lambda$.
    
    High probability upperbound $\bar{\epsilon}$ on the error $\|\widehat{g}(\theta) - g(\theta)\|_*$.
    
    Upperbound $\bar{s}$ on the sparsity $s$.
    
    \item Set $R_0 = R$ and $\tau = \frac{10\sqrt{8\bar{s}} \bar{\epsilon}}{\kappa}$ and $R^\star = \frac{80 \lambda \bar{s}\bar{\epsilon}}{\kappa}$.
    \item Set $k=0$ and the per stage number of iterations $T' = \Big\lceil \Big(\frac{\nu + M^2}{\bar{\epsilon}}\Big)^2 \Big\rceil$
    \item For $k = 1,\dots, K$ :
    \begin{itemize}
        \item Set $\theta^{(k)}_0 = \theta^{(k-1)}$ and $\Theta_k = B_{\|\cdot\|}(\theta^{(k)}_0, R_{k-1}).$
        \item Run Dual averaging with prox-function $\omega_{\theta^{(k-1)}}$ and  steps $a_i = R_{k-1}$ for $T'$ iterations.
        \item Set $\theta^{(k)} = \sparse_{\bar{s}} (\widetilde{\theta}^{(k)})$ where $\widetilde{\theta}^{(k)} := \theta_{T'}^{(k)}$.
        \item Set $R_{k} =\max\big( \tau R_{k-1}, \frac{1}{2}(R_{k-1} + R^\star) \big)$.
    \end{itemize}
    \item \textit{Output}: The final stage estimate $\theta^{(K)}$.
\end{itemize}
Similar to AMMD, the AMDA algorithm runs multiple optimisation stages through which the parameter space is repeatedly restricted allowing to obtain similar benefits regarding convergence speed and statistical performance. However, it is worth noting that these improvements are obtained under much milder conditions here since the objective may not even be smooth and is only required to satisfy the pseudo-linear growth condition of Assumption~\ref{asm:pseudolingrowth} whereas smoothness and quadratic minorization or strong convexity were indispensable in previous works~\cite{juditsky2011first, juditsky2014deterministic, Juditsky2020SparseRB}.

As previously mentioned for AMMD, a simplified version of AMDA can be implemented which does not require knowledge of all the involved quantities, see Section~\ref{sec:exp} for details. We state the convergence guarantees for the above algorithm in the following Theorem.

\begin{theorem}\label{thm:DAoptimalrate}
Grant Assumptions~\ref{asm:sparse},~\ref{asm:data},~\ref{asm:lipschitz},~\ref{asm:pseudolingrowth} and assume that $\tau = \frac{10\sqrt{8\bar{s}} \bar{\epsilon}}{\kappa} < 1$. At the end of each stage $k \geq 1$, we have \textup:
\begin{equation}\label{eq:thmDAparamIneq}
    \big\|\theta^{(k)} - \theta^{\star}\big\| \leq \sqrt{2\bar{s}}\big\|\theta^{(k)} - \theta^{\star}\big\|_2 \leq 2\sqrt{2\bar{s}}\big\|\widetilde{\theta}^{(k)} - \theta^{\star}\big\|_2 \leq R_k,
\end{equation}
\begin{equation}\label{eq:thmDAobjIneq}
    \cL(\widetilde{\theta}^{(k)}) - \cL(\theta^\star) \leq 5\bar{\epsilon}R_{k-1}.
\end{equation}
Moreover, the total number of necessary iterations before $R_k \leq 2R^\star = \frac{160 \lambda \bar{s}\bar{\epsilon}}{\kappa}$ is at most
\begin{equation*}
    \log(R_0 / R^\star) \Big(\frac{1}{\log(1/\tau)} + \frac{1}{\log(2)}\Big)\Big( \Big(\frac{\nu + M^2}{\bar{\epsilon}}\Big)^2 +1\Big).
\end{equation*}

\end{theorem}
The proof of Theorem~\ref{thm:DAoptimalrate} is given in Appendix~\ref{proof:thmDA} and shows that the optimization stages go through two phases: an initial \textit{linear} phase corresponding to the linear regime of the lower bound given by Assumption~\ref{asm:pseudolingrowth} and a later \textit{quadratic} phase during which the quadratic regime takes over. The success of the linear phase relies on the condition $\tau < 1$ which can be rewritten as $\bar{\epsilon} \leq O(\bar{s} \kappa)$ where the factor $\bar{s}$ is a byproduct of measuring the parameter error with the norm $\|\cdot\|$ while Assumption~\ref{asm:pseudolingrowth} is stated with the Euclidean one. In the linear regime, $\kappa$ acts as a lower bound for the gradient norm so that the condition ensures that the error is smaller than the actual gradient allowing the optimisation to make progress. Theorem~\ref{thm:DAoptimalrate} states that convergence to the optimum occurs at geometrical speed through the stages of AMDA despite the absence of strong convexity. This is achieved thanks to the choice of step $a_i = R_{k-1}$ in AMDA which leads to a bound in terms of $R$ rather than $R^2$ emerging from the term $\omega(\theta^\star)$ in Proposition~\ref{prop:DA}, see the proof for more details.

\paragraph{Remark} In recent work,~\cite{juditsky2019unifying} created a common framework for the study of the Mirror Descent and Dual Averaging algorithms which they recover as special cases of a generic Unified Mirror Descent procedure. However, the distinction between the two remains necessary since they address smooth and non-smooth objectives respectively and, in each case, the attainable within-stage convergence speed differs from $1/t$ to $1/\sqrt{t}$ as seen in Propositions~\ref{prop:correctedMD} and~\ref{prop:DA}. This is reflected in Theorems~\ref{thm:MDoptimalrate} and~\ref{thm:DAoptimalrate} which display a dependence of the necessary number of iterations in $1/\bar{\epsilon}$ in the gradient error for Mirror Descent as opposed to $1/\bar{\epsilon}^2$ for Dual Averaging.

\section{Applications}\label{sec:applis}

We now consider a few problems which may be solved using the previous optimization procedures. As said earlier, we have omitted to quantify the gradient errors $\|\epsilon\|_*$ until now. This is because the definition of the dual norm $\|\cdot\|_*$ is problem dependent. In the next subsections, we consider a few instances and propose adapted gradient estimators for them. In each case, the existence of a second moment for the gradient random variable $G(\theta) := \ell'(\langle \theta, X \rangle, Y)X$ is required. This follows from the next Lemma proven in Appendix~\ref{proof:obj-grad-moment} based on Assumption~\ref{asm:data}.
\begin{lemma}
    \label{lem:obj-grad-moment}
    Under Assumption~\ref{asm:data} the objective $\cL(\theta)$ is well defined for all $\theta \in \Theta$ and we have \textup:
    \begin{equation*}
        \E\big[ \big\| G(\theta) \big\|_*^2 \big] = \E\big[ \big\|\ell'(X^\top \theta, Y) X \big\|_*^2 \big] < +\infty.
    \end{equation*}
\end{lemma}
In what follows, we will assume that, at each step of the optimization algorithm, the estimation of the gradient is performed with a new batch of data. For example, if the available data set contains $n$ samples then it needs to be divided into $T$ disjoint splits in order to make $T$ optimization steps. This is necessary in order to guarantee that the gradient samples used for estimation at each step $t$ are independent from $\theta_t$, the (random) current parameter which depends on the data used before.

This trick was previously used for example in~\cite{HeavyTails} for the same reasons. A possible alternative is to use an $\epsilon$-net argument or Rademacher complexity in order to obtain uniform deviation bounds on gradient estimation over a compact parameter set $\Theta$. However, this entails extra dependence on the dimension in the resulting deviation bound which we cannot afford in the high-dimensional setting. For these reasons, we prefer to use data splitting in this work and regard it more like a proof artifact rather than a true practical constraint. Note that we do not implement it later in our experimental section.

\subsection{Vanilla sparse estimation}\label{sec:appliVanillaSparse}

In this section, we consider the problem of optimizing an objective $\cL (\theta) = \E [\ell(\langle \theta, X \rangle , Y)]$ where the covariate space $\cX$ is simply $\R^d$ and the labels are either real numbers $\cY = \R$ (regression) or binary labels (binary classification). In this case, the parameter space is a subset $\Theta \subset \R^d$, the sparsity of a parameter $\theta \in \Theta$ is measured as its number of nonzero entries $S(\theta) = \sum_{j\in\setint{d}}\ind{\theta_j\neq 0}$, and $\|\cdot\|$ is defined as the $\ell_1$ norm $\|\cdot\| = \|\cdot\|_1$ so that $\|\cdot\|_* = \|\cdot\|_{\infty}$. We define the distance generating function $\omega$ as :
\begin{equation*}
    \omega(\theta) = \frac{1}{2}e\log(d)d^{(p-1)(2-p)/p} \|\theta\|_p^2 \quad \text{with }\quad p = 1 + \frac{1}{\log (d)}.
\end{equation*}
One can check that the above definition satisfies the requirements of Definition~\ref{def:dgf}. In particular, it is strongly convex w.r.t. $\|\cdot\|$ and quadratically growing with constant $\nu = \frac 1 2 e^2 \log(d)$ (see~\cite[Theorem 2.1]{nesterov2013first}). In others words, conditions 1 and 3 of Definition~\ref{def:dgf} are reconciled with $\nu = O(\log d).$ For the sake of achieving this compromise, the previous choice of distance generating function is common in the high-dimensional learning literature using mirror descent~\cite{agarwal2012stochastic, duchi2010composite, shalev2011stochastic, nemirovskij1983problem, juditsky2011first, nesterov2013first} up to slight variations in $p$ and the multiplying factor. 

We consider Assumption~\ref{asm:data} on the data with a constant fraction of outliers $|\cO| \leq \eta n$ for some $\eta < 1/2$ ($\eta$-corruption) so that the gradient samples $g^i(\theta) := \ell'(\theta^\top X_i, Y_i) X_i$ may be both heavy-tailed and corrupted as well. We propose to compute $\widehat{g}(\theta)$ as the coordinatewise trimmed mean of the sample gradients i.e. 
\begin{equation}\label{eq:coordTMestimator}
    \widehat{g}_j(\theta) = \mathtt{TM}_{\alpha} \big(g_j^1(\theta), \dots, g_j^n(\theta)\big),
\end{equation}
where, assuming without loss of generality that $n$ is even, the trimmed mean estimator with parameter $\alpha$ for a sample $x_1,\dots, x_n \in \R$ is defined as follows
\begin{equation*}
    \mathtt{TM}_{\alpha}(x_1, \dots, x_n) = \frac{2}{n} \sum_{i = n/2 +1}^n
    q_\alpha \vee x_i \wedge q_{1 - \alpha},
\end{equation*}
where we denoted $a \wedge b:= \min(a, b)$ and $a \vee b:= \max(a, b)$ and used the quantiles $q_\alpha := x^{([\alpha n/2])}$ and $q_{1 - \alpha} = x^{([(1-\alpha) n/2])}$ with $x^{(1)} \leq \cdots \leq x^{(n/2)}$ the order statistics of $(x_i)_{i\in \setint{n/2}}$ and where $[\cdot]$ denotes the integer part. 

The main hurdle to compute the trimmed mean estimator is to find the two previous quantiles. A naive approach for this task would be to sort all the values leading to an $O(n\log(n))$ complexity. However, this can be brought down to $O(n)$ using the median-of-medians algorithm (see for instance~\cite[Chapter 9]{cormen2009introduction}) so that the whole procedure runs in linear time.

We now give the deviation bound satisfied by the estimator~\eqref{eq:coordTMestimator}. 
We denote $x^{j}$ as the $j$-th coordinate of a vector $x$.

\begin{lemma}\label{lem:coordTM}
    Grant Assumption~\ref{asm:data} with a fraction of outliers $|\cO| \leq \eta n$ with $\eta < 1/8$. Fix $\theta \in \Theta$, let $\sigma_j^2 = \Var (\ell'(\theta^\top X, Y) X^j)$ for $j\in \setint{d}$ be the gradient coordinate variances and let $1> \delta > e^{-n/2}/4$ be a failure probability and consider the coordinatewise trimmed mean estimator~\eqref{eq:coordTMestimator} with parameter $\alpha = 8\eta + 12\frac{\log(4/\delta)}{n}$. Denoting $\sigma_{\max}^2 = \max_j \sigma_j^2$, we have with probability at least $1 - \delta$ :
    \begin{equation}
        \big\|\widehat{g}(\theta) - g(\theta)\big\|_{\infty} \leq 7 \sigma_{\max} \sqrt{4\eta + 6\frac{\log(4/\delta) + \log(d)}{n}}.
    \end{equation}
\end{lemma}
\begin{proof}
This is an almost immediate application of~\cite[Lemma 9]{gaiffas2022robust} (see also~\cite[Theorem 1]{lugosi2021robust}). By an immediate application of the latter, we obtain for each $j\in\setint{d}$ that with probability at least $1 - \delta/d$ we have :
    \begin{equation}
        \big|\widehat{g}_j(\theta) - g_j(\theta)\big| \leq 7 \sigma_{j} \sqrt{4\eta + 6\frac{\log(4/\delta) + \log(d)}{n}}.
    \end{equation}
    Hence, the lemma follows by a simple union bound argument.
\end{proof}
For the sake of simplicity, this deviation bound is only stated for a square integrable gradient which yields a $\sqrt{\eta}$ dependence in the corruption rate. More generally, for a random variable admitting a finite moment of order $k$, one can derive a bound in terms of $\eta^{1 - 1/k}$ which reflects a milder dependence for greater $k$, see~\cite[Lemma 9]{gaiffas2022robust} for the bound in question.

In a way, the fact that the gradient error is measured with the infinity norm in this setting is a ``stroke of luck'' since the optimal dependence in the dimension for the statistical error becomes achievable using only a univariate estimator. This is in contrast with situations where multivariate robust estimators need to be used for which the combination of efficiency, sub-Gaussianity and robustness to $\eta$-corruption is hard to come by.

Based on Lemma~\ref{lem:coordTM} we obtain a gradient error of order $\bar{\epsilon} = O(\sqrt{\log(d)/n})$. Plugging this deviation rate into Theorem~\ref{thm:MDoptimalrate} yields the optimal $s\log(d)/n$ rate for vanilla sparse estimation. The same applies for Theorem~\ref{thm:DAoptimalrate} provided the condition $\tau < 1$ holds.

\begin{corollary}
In the context of Theorem~\ref{thm:MDoptimalrate} and Lemma~\ref{lem:coordTM}, let the AMMD algorithm be run starting from $\theta_0\in\Theta = B_{\|\cdot\|}(\theta_0, R)$ using the coordinatewise trimmed mean estimator with sample splitting i.e. at each iteration a different batch of size $\widetilde{n} = n/T$ is used for gradient estimation with confidence $\widetilde{\delta} = \delta/T$ where $T$ is the total number of iterations. Let $K$ be the number of stages and $\widehat{\theta}$ the obtained estimator. Denote $\sigma_{\max}^2 = \sup_{\theta \in \Theta} \max_{j\in\setint{d}}\Var (\ell'(\theta^\top X, Y) X^j)$, with probability at least $1 - \delta$, the latter satisfies :
    \begin{equation*}
        \big\|\widehat{\theta} - \theta^\star\big\|_2 \leq \frac{2^{-K/2}R}{\sqrt{\bar{s}}} + \frac{140\sqrt{2\bar{s}}\sigma_{\max}}{\kappa}\sqrt{4\eta + 6 \frac{\log(4/\widetilde{\delta}) + \log(d)}{\widetilde{n}}}.
    \end{equation*}
\end{corollary}
\begin{proof}
The result is easily obtained by combining Theorem~\ref{thm:MDoptimalrate} and Lemma~\ref{lem:coordTM} with a union bound argument over all iterations $T$ in order to bound $\bar{\epsilon} = \max_{t=0,\dots, T-1}\|\epsilon_t\|_*$ as defined in Proposition~\ref{prop:correctedMD}.
\end{proof}
In the above upper bound, the optimisation error vanishes exponentially with the number of stages $K$ so that the final error can be attributed in large part to the second statistical error term. The latter achieves the nearly optimal $\sqrt{s\log(d)/n}$ rate and combines robustness to heavy tails and $\eta$-corruption. Moreover, this statement holds for a \emph{generic} loss function satisfying the assumptions of Section~\ref{sec:md}. These can be further weakened to those given in Section~\ref{sec:da} by using Dual Averaging while preserving the same statistical rate. The statement of a result in this weaker setting is postponed to Section~\ref{sec:DAforVanilla} of the Appendix in order to avoid excessive repetition. 
To our knowledge, this is the first result with such properties for vanilla sparse estimation whereas previous results from the literature either focused on specific learning problems with a fixed loss function~\cite{dalalyan2019outlier, sasai2022robust} or isolated the issues of robustness by assuming the data to be either heavy tailed or Gaussian and corrupted~\cite{balakrishnan2017computationally, liu2019high, liu2020high, Juditsky2020SparseRB}. Furthermore, we stress that this error bound is achieved at a comparable computational cost to that of a standard non robust algorithm since, as mentioned earlier, the robust trimmed mean estimator can be computed in linear time.

A possible room for improvement is to try to remove the $\bar{s}$ factor multiplying the corruption rate $\eta$. The only works we are aware of achieving this are~\cite{liu2020high, sasai2022robust} but both involve costly data-filtering steps. We suspect this may be an inevitable price to pay for such an improvement.

\subsection{Group sparse estimation}

In the group sparse case, we again consider the covariate space $\cX = \R^d$ where the coordinates $\setint{d}$ are arranged into groups $G_1, \dots, G_{N_G}$ which form a partition of the coordinates $\setint{d}$ and sparsity is measured in terms of these groups i.e. $S(\theta) = \sum_{j\in\setint{N_G}}\ind{\theta_{G_j} \neq 0}$ and we assume the optimal $\theta^\star = \argmin_{\theta}\cL(\theta)$ satisfies $S(\theta^\star) \leq s$. 
The norm $\|\cdot\|$ is set to be the $\ell_1/\ell_2$ norm : $\|\theta\|_{1, 2} = \sum_{j\in\setint{N_G}} \|\theta_{G_j}\|_2$ and the dual norm is the analogous $\ell_{\infty}/\ell_2$ norm $\|\theta\|_{\infty, 2} = \max_{j\in\setint{N_G}} \|\theta_{G_j}\|_2$. The label set $\cY$ may be equal to $\R$ (regression) or a finite set (binary or multiclass classification).

For simplicity, we assume that the groups are of equal size $m$ so that $d = m N_G$. This is for example the case when trying to solve a $d$-dimensional linear multiclass classification with $K$ classes by estimating a parameter $\theta\in \R^{d\times K}$ and predicting $\argmax_j (\theta^\top X)_j$ for a datapoint $X \in \R^d$. In this case, it makes sense to consider the rows $(\theta_{i,:})_{i\in\setint{d}}$ as groups which are collectively determined to be zero or not depending on the importance of feature $i$. For simplicity, we restrict ourselves to this setting until the end of this section with no loss of generality.

Analogously to the vanilla sparse case, following~\cite{nesterov2013first}, the distance generating function (or prox-function) $\omega$ may be chosen in this case as:
\begin{equation*}
    \omega(\theta) = \frac{1}{2}e\log(d)d^{(p-1)(2-p)/p} \Big(\sum_{i=1}^{d}\|\theta_{i,:}\|_2^p\Big)^{2/p} \quad \text{with }\quad p = 1 + \frac{1}{\log (d)}.
\end{equation*}

We assume the data corresponds to Assumption~\ref{asm:data} with $\eta$-corruption (i.e. $|\cO| \leq \eta n$). Analogously to the vanilla case, we propose to estimate the gradient \emph{groupwise} i.e. one group of coordinates at a time. For this task, a multivariate, sub-Gaussian and corruption-resilient estimation algorithm is needed. We suggest to use the estimator advocated in \cite{diakonikolas2020outlier} for this purpose which remarkably combines these qualities. We refer to it as the DKP estimator and restate its deviation bound here for the sake of completeness.

\begin{proposition}[{\cite[Proposition 1.5]{diakonikolas2020outlier}}]\label{prop:dkk}
Let $T$ be an $\eta$-corrupted set of $n$ samples from a distribution $P$ in $\R^d$ with mean $\mu$ and covariance $\Sigma$.
Let $\eta' = \Theta(\log(1/\delta )/n + \eta) \leq c$ be given\textup, for a constant $c > 0$. Then any stability-based algorithm on input $T$ and $\eta'$\textup, efficiently computes $\widehat{\mu}$ such that with probability at least $1 - \delta$\textup, we have \textup:
\begin{equation}
    \big\|\widehat{\mu} - \mu\big\|_2 = O\bigg(\sqrt{\frac{\Tr(\Sigma) \log\mathrm{r}(\Sigma)}{n}} + \sqrt{\|\Sigma\|_{\op} \eta} + \sqrt{\frac{\|\Sigma\|_{\op}\log(1/\delta)}{n}}\bigg),
\end{equation}
where $\mathrm{r}(\Sigma) = \Tr(\Sigma)/\|\Sigma\|_{\op}$ is the stable rank of $\Sigma$.
\end{proposition}
The above bound is almost optimal up to the $\sqrt{\log\mathrm{r}(\Sigma)}$ factor which is at most $\sqrt{\log(d)}$. Note that we are also aware that \cite[Proposition 1.6]{diakonikolas2020outlier} states that, by adding a Median-Of-Means preprocessing step, stability based algorithms can achieve the optimal deviation. Nevertheless, the number $k$ of block means required needs to be such that $k \geq 100\eta n$ so that the corruption rate $\eta$ is strongly restricted because necessarily $n \geq k$. Therefore, we prefer to stick with the result above.

An algorithm with the statistical performance stated in Proposition~\ref{prop:dkk} is given, for instance, in~\cite[Appendix A.2]{diakonikolas2020outlier}, we omit it here for brevity. Now, we can reuse the previous section's trick by estimating the gradients \emph{blockwise} this time to obtain the following lemma:

\begin{lemma}\label{lem:groupDKK}
    Grant Assumption~\ref{asm:data} with a fraction of outliers $|\cO| \leq \eta n$. Fix $\theta \in \Theta$ and denote $G_1(\theta), \dots, G_n(\theta)$ the gradient samples distributed according to $G(\theta) \in \R^{d\times K}$ (except for the outliers). Let $\Sigma_j = \Var (G(\theta)_{j,:}) \in \R^{K \times K}$ be the gradient block variances. Consider the groupwise estimator $\widehat{g}(\theta)$ defined such that $\widehat{g}(\theta)_{j,:}$ is the DKP estimator applied to $G(\theta)_{j,:}$. Then we have with probability at least $1 - \delta$ :
    \begin{equation}
        \big\|\widehat{g}(\theta) - g(\theta)\big\|_{\infty, 2} \leq O\bigg(\max_j \sqrt{\frac{\Tr(\Sigma_j) \log\mathrm{r}(\Sigma_j)}{n}} + \sqrt{\|\Sigma_j\|_{\op}}\bigg( \sqrt{\eta} + \sqrt{\frac{\log(1/\delta) + \log(d)}{n}} \bigg)\bigg) \label{eq:dkkDeviation}
    \end{equation}
\end{lemma}
\begin{proof}
Inequality~\eqref{eq:dkkDeviation} is straightforward to obtain using Proposition~\ref{prop:dkk} and a union bound argument on $j\in\setint{d}$.
\end{proof}
One can easily see that, in the absence of corruption, the above deviation bound scales as $\widetilde{O}\Big(\sqrt{\frac{K}{n}} + \sqrt{\frac{\log (d)}{n}} \Big)$. Combined with the sparsity assumption given above, plugging this estimation, which applies for the gradient error $\|\epsilon_t\|_*$, into Theorems~\ref{thm:MDoptimalrate} or~\ref{thm:DAoptimalrate} yields near optimal (up to a logarithmic factor) estimation rates for the group-sparse estimation problem~\cite{negahban2012unified, lounici2009taking}. The following corollary formalizes this statement.
\begin{corollary}
In the context of Theorem~\ref{thm:MDoptimalrate} and Lemma~\ref{lem:groupDKK}, let the AMMD algorithm be run starting from $\theta_0\in\Theta = B_{\|\cdot\|}(\theta_0, R)$ and using the blockwise DKP estimator with sample splitting i.e. at each iteration a different batch of size $\widetilde{n} = n/T$ is used for gradient estimation with confidence $\widetilde{\delta} = \delta/T$ where $T$ is the total number of iterations. Let $N$ be the number of stages and $\widehat{\theta}$ the obtained estimator. With probability at least $1 - \delta$, the latter satisfies :
\begin{align*}
    \big\|\widehat{\theta} - \theta^\star\big\|_2 &\leq \frac{2^{-N/2}R}{\sqrt{\bar{s}}} + \frac{\sqrt{\bar{s}}}{\kappa} O\bigg(\sup_{\theta\in\Theta}\max_{j\in\setint{d}} \sqrt{\frac{\Tr(\Sigma_{\theta, j}) \log\mathrm{r}(\Sigma_{\theta, j})}{\widetilde{n}}}\: + \\
    &\quad \sqrt{\|\Sigma_{\theta, j}\|_{\op}}\bigg( \sqrt{\eta} + \sqrt{\frac{\log(1/\widetilde{\delta}) + \log(d)}{\widetilde{n}}} \bigg)\bigg) \\
    &\leq \frac{2^{-N/2}R}{\sqrt{\bar{s}}} + \frac{\sqrt{\bar{s}}}{\kappa}\widetilde{O}\bigg(\sup_{\theta\in\Theta}\max_{j\in\setint{d}} \sqrt{\frac{K}{\widetilde{n}}} + \bigg( \sqrt{\eta} + \sqrt{\frac{\log(1/\widetilde{\delta}) + \log(d)}{\widetilde{n}}} \bigg)\bigg),
\end{align*}
where $\Sigma_{\theta, j} = \Var(G(\theta)_{j,:})$.
\end{corollary}
As before, the stated bound reflects a linearly converging optimisation and displays a statistical rate nearly matching the optimal rate for group-sparse estimation~\cite{negahban2012unified, lounici2009taking} up to logarithmic factors. In addition, robustness to heavy tails and $\eta$-corruption likely makes this result the first of its kind for group-sparse estimation since all robust works we are aware of focus on vanilla sparsity.

\subsection{Low-rank matrix recovery}\label{sec:appliLowRank}

We also consider the variant of the problem where the covariates belong to a matrix space $\cX = \R^{p\times q}$ in which case the objective $\cL (\theta) = \E [\ell(\langle \theta, X \rangle , Y)]$ needs to be optimized over $\Theta \subset \R^{p\times q}$. In this setting, $\langle \cdot, \cdot \rangle$ refers to the Frobenius scalar product between matrices 
\begin{equation*}
    \langle a, b \rangle = \Tr(a^\top b).
\end{equation*}
Without loss of generality, we assume that $p \geq q$ and sparsity is meant as the number of non zero singular values i.e. for a matrix $A\in \R^{p\times q}$, denoting $\sigma(A) = (\varsigma_j(A))_{j\in\setint{q}}$ the set of its singular values we define $S(A) = \sum_{j\in\setint{q}}\ind{\varsigma_j(A) \neq 0}$. We set $\|\cdot\|$ to be the nuclear norm $\|A\| = \|\sigma(A)\|_1$ and the associated dual norm is the operator norm $\|\cdot\|_* = \|\cdot\|_{\op}$. 

On the optimization side, an appropriate distance generating function (resp. prox-function) needs to be defined for this setting before Mirror Descent (resp. Dual Averaging) can be run. Based on previous literature (see \cite[Theorem 2.3]{nesterov2013first} and~\cite{Juditsky2020SparseRB}), we know that the following choice satisfies the requirements of Definition~\ref{def:dgf} :
\begin{equation*}
    \omega(\theta) = 2e \log(2q) \Big( \sum_{j=1}^q \varsigma_j(\theta)^{1+r}\Big)^{2/(1+r)} \quad \text{with}\quad r = 1/(12 \log(2q)).
\end{equation*}
This yields a corresponding quadratic growth parameter $\nu = O(\log(q))$. In order to fully define our optimization algorithm for this problem, it remains to specify a robust estimator for the gradient. This turns out to be a challenging question since the estimated value is matricial and the operator norm $\|\cdot\|_* = \|\cdot\|_{\op}$ emerging in this case is a fairly exotic choice to measure statistical error.

In order to achieve a nearly optimal statistical rate we define a new estimator called ``CM-MOM'' (short for Catoni Minsker Median-Of-Means) which combines methods from~\cite{minsker2018sub} for sub-Gaussian matrix mean estimation and ideas from~\cite{minsker2015geometric, hsu2016loss} in order to apply a Median-Of-Means approach for multivariate estimation granting robustness to outliers provided these are limited in number. We now define this estimator in detail. Let $\psi$ be a function defined as
\begin{equation*}
    \psi(x) = \log(1 + |x| + x^2/2)
\end{equation*}
We consider a restricted version of Assumption~\ref{asm:data} in which the number of outliers is limited as\footnote{In fact, one may allow up to $|\cO|\leq K/2$ outliers at the price of worse constants in the resulting deviation bound. See the proof of Proposition~\ref{prop:spectralMOM}} $|\cO| \leq K/12$ where $K$ is an integer such that $K < n$. Provided a sample of matrices $A_1, \dots, A_n \in \R^{p\times q}$ and a scale parameter $\chi > 0$, the CM-MOM estimator proceeds as follows :
\begin{itemize}
    \item Split the sample into $K$ disjoint blocks $B_1, \dots, B_K$ of equal size $m = n/K$.
    \item Compute the dilated block means $\xi^{(j)}$ for $j=1,\dots, K$ as 
    \begin{equation*}
        \xi^{(j)} = \frac{1}{\chi m}\sum_{i\in B_j} \psi(\theta \widetilde{A}_i) \in \R^{(p+q)\times (p+q)},
    \end{equation*}
    where the dilation $\widetilde{A}$ of matrix $A\in \R^{p\times q}$ is defined as $\widetilde{A} = \begin{pmatrix}0 & A \\ A^\top & 0 \end{pmatrix} \in \R^{(p+q)\times (p+q)}$ which is symmetric and the function $\psi$ is applied to a symmetric matrix $S\in\R^{d\times d}$ by applying it to its eigenvalues i.e. let $S = UDU^\top$ be its eigendecomposition with $D=\diag((\lambda_j)_{j\in\setint{d}})$ then $\psi(S) = U\psi(D)U^\top = U\diag((\psi(\lambda_j))_{j\in\setint{d}})U^\top.$
    \item Extract the block means $\widehat{\mu}_j \in \R^{p\times q}$ such that $\xi^{(j)} = \begin{pmatrix}\xi^{(j)}_{11} & \widehat{\mu}_j \\ \widehat{\mu}_j^\top & \xi^{(j)}_{22} \end{pmatrix}.$
    \item Compute the pairwise distances $r_{jl} = \|\widehat{\mu}_j  - \widehat{\mu}_l\|_{\op}$ for $j,l \in \setint{K}$.
    \item Compute the vectors $r^{(j)} \in \R^K$ for $j\in \setint{K}$ where $r^{(j)}$ is the increasingly sorted version of $r_{j:}$.
    \item return $\widehat{\mu}_{\widehat{i}}$ where $\widehat{i} \in \argmin_i r^{(i)}_{K/2}$.
\end{itemize}
One may guess that the choice of the scale parameter $\chi$ plays an important role to guarantee the quality of the estimate. This aspect is inherited from Catoni's original estimator for the mean of a heavy-tailed real random variable~\cite{catoni2012challenging} from which Minsker's estimator~\cite{minsker2018sub}, which we use to estimate the block means, is inspired. The following statement gives the optimal value for $\chi$ and the associated deviation bound satisfied by CM-MOM.

\begin{proposition}[CM-MOM]\label{prop:spectralMOM}
    Let $A_1,\dots, A_n\in \R^{p\times q}$ be an i.i.d sample following a random variable $A$ with expectation $\mu = \E A$ such that a subset of indices $\cO \subset \setint{n}$ are outliers and finite variance
    \begin{equation*}
        v(A) = \max \big( \big\|\E (A - \mu)(A - \mu)^\top \big\|_{\op}, \big\|\E (A - \mu)^\top(A - \mu)\big\|_{\op} \big)< \infty.
    \end{equation*}
    Let $\delta > 0$ be a failure probability and take $K = \lceil 18\log(1/\delta) \rceil < n$ blocks\textup, we assume $n = mK$. Let $\widehat{\mu}$ be the CM-MOM estimate as defined above with scale parameter
    \begin{equation*}
        \chi = \sqrt{\frac{2m\log(8(p+q))}{v(A)}}.
    \end{equation*}
    Assume we have $|\cO| \leq K/12$ then with probability at least $1 - \delta$ we have \textup:
    \begin{equation*}
        \big\| \widehat{\mu} - \mu \big\|_{\op} \leq 18\sqrt{\frac{v(A)\log(8(p+q))\log(1/\delta)}{n}}.
    \end{equation*}
\end{proposition}
Proposition~\ref{prop:spectralMOM} is proven in Appendix~\ref{proof:spectralMOM} and enjoys a deviation rate which scales optimally, up to logarithmic factors, as $\sqrt{p+q}$ in the dimension~\cite{vershynin2018high}. This dependence is hidden by the factor $\sqrt{v(A)}$ which scales in that order (see for instance~\cite{tropp2015introduction}). Although the dependence of the optimal scale $\chi$ on the unknown value of $v(A)$ constitutes an obstacle, previous experience using Catoni-based estimators~\cite{catoni2012challenging, pmlr-v97-holland19a, gaiffas2022robust} has shown that the choice is lenient and good results are obtained as long as a value of the correct scale is used. Possible improvements for Proposition~\ref{prop:spectralMOM} are to derive a bound with an additive instead of multiplicative term $\log(1/\delta)$ or supporting $\eta$-corruption. However, we are not aware of a more robust solution for matrix mean estimation in the general case than the above result. 

Now that we have an adapted gradient estimation procedure, we can proceed to combine its deviation bound with our optimization theorems in order to obtain guarantees on learning performance.

\begin{corollary}\label{cor:lowrankmatrix}
In the context of Theorem~\ref{thm:MDoptimalrate} and Proposition~\ref{prop:spectralMOM}, let the AMMD algorithm be run starting from $\theta_0\in\Theta = B_{\|\cdot\|}(\theta_0, R)$ and using the CM-MOM estimator with sample splitting i.e. at each iteration a different batch of size $\widetilde{n} = n/T$ is used for gradient estimation with confidence $\widetilde{\delta} = \delta/T$ where $T$ is the total number of iterations. Assume that each batch contains no more than $K/12$ outliers. Let $N$ be the number of stages and $\widehat{\theta}$ the obtained estimator. With probability at least $1 - \delta$, the latter satisfies :
\begin{align*}
    \big\|\widehat{\theta} - \theta^\star\big\|_2 &\leq \frac{2^{-N/2}R}{\sqrt{\bar{s}}} + \sup_{\theta \in \Theta}\frac{360\sqrt{\bar{s}}}{\kappa}\sqrt{\frac{2v(G(\theta)) \log(8(p+q)) \log(1/\widetilde{\delta})}{\widetilde{n}}},
\end{align*}
where $\|\cdot\|_2$ denotes the Frobenius norm.
\end{corollary}
Corollary~\ref{cor:lowrankmatrix} matches the optimal performance bounds given in classical literature for low-rank matrix recovery~\cite{koltchinskii2011nuclear, rohde2011estimation, candes2011tight, negahban2011estimation} up to logarithmic factors. The previous statement is most similar to~\cite[Proposition 3.3]{Juditsky2020SparseRB} except that it applies for more general learning tasks and under much lighter data assumption.

\section{Implementation and Numerical Experiments}\label{sec:exp}

In this section, we demonstrate the performance of the proposed algorithms on synthetic and real data. Before we proceed, we prefer to point out that our implementation does not \emph{exactly} correspond to the previously given pseudo-codes. Indeed, as the reader may have noticed previously, certain instructions of AMMD and AMDA require the knowledge of quantities which are not available in practice and, even in a controlled setting, the estimation of some quantities (such as the maximum gradient error $\bar{\epsilon}$) may be overly conservative which generally impedes the proper convergence of the optimization. We list the main divergences of our implementation from the theoretically studied procedures of AMMD and AMDA given before :
\begin{enumerate}
    \item For AMMD, we only use the conventional $\prox$ operator rather than the corrected $\widehat{\prox}$ operator defined in Section~\ref{sec:md}.
    \item For both AMMD and AMDA, the whole data set is used at each step to compute a gradient estimate and no data-splitting is performed.
    \item The radii $R_k$ are taken constant equal to a fixed $R > 0$.
    \item The stage lengths ($T_k$ for AMMD and $T'$ for AMDA) are fixed as constants.
    \item The number of stages is determined through a maximum number of iterations but the algorithm stops after the last whole stage.
    \item The within-stage step-sizes $a_i$ in AMDA are fixed to a small constant (smaller than $R$) for more stability.
\end{enumerate}
The constant stage-lengths are fixed using the following heuristic: run the MD/DA iteration while tracking the evolution of the empirical objective on a validation subset of the data\footnote{To remain consistent with a robust approach, the objective is estimated using a trimmed mean here as well.} and set the stage-length as the number of steps before a plateau is reached. Reaching a plateau indicates that the current reference point $\theta^{(k)}$ has become too constraining for the optimization and more progress can be made after updating it.

The simplifications brought by points 1.~and 2.~are related to likely proof artifacts. Indeed as pointed out in Section~\ref{sec:md} the use of the corrected $\widehat{\prox}$ operator rather than simply $\prox$ is chiefly meant to ensure objective monotonicity while the data splitting ensures the gradient deviation bounds are usable in the proofs at each iteration.

Points 3.~and 4.~are due to the fact that the values of the stage-lengths $T_k$/$T'$ and the radii $R_k$ used in AMMD and AMDA are based on conservative estimates from the theoretical analysis making them unfit for practical implementation. Moreover, the said estimates use constants which cannot be identified in an arbitrary setting. For example, for least squares regression, the quadratic minorization constant $\kappa$ depends on the data distribution which is unknown in general. 

Finally, point 5.~is commonplace for batch learning where one may simply iterate until convergence and point 6.~follows the wisdom that smaller step-sizes ensure more stability.

Despite these differences, the numerical experiments we present below demonstrate that our implementations perform on par with the associated theoretical results.

Clearly, optimisation using Mirror Descent should be preferred over Dual Averaging due to its faster convergence speed. This is conditioned by the smoothness of the objective $\cL$ which holds, for example, when the loss function $\ell$ is smooth. 
If $\ell$ is not smooth but the data distribution 
contains no atoms, one may still use Mirror Descent since it is reasonable to expect the objective $\cL$ to be smoothed by the expectation~\eqref{eq:objective}. 
Note however that such an objective is likely not to satisfy Assumption~\ref{asm:quadgrowth} making Theorem~\ref{thm:MDoptimalrate} inapplicable. Still, in this case, if the weaker Assumption~\ref{asm:pseudolingrowth} holds, the expected performance is as stated in Theorem~\ref{thm:DAoptimalrate} with an improved number of required iterations of order $1/\bar{\epsilon}$ instead of $1/\bar{\epsilon}^2$ due to faster within-stage optimisation.

\subsection{Synthetic sparse linear regression}

We first test our algorithms on the classic problem of linear regression. We generate $n$ covariates $X_i \in \R^d$ following a non-isotropic distribution with covariance matrix $\Sigma$ and labels  $Y_i = X_i^\top \theta^\star + \xi_i$ for a fixed $s$-sparse $\theta^\star \in \R^d$ and simulated noise entries $\xi_i$. The covariance matrix $\Sigma$ is diagonal with entries drawn uniformly at random in $[1, 10]$.

We use the least-squares loss $\ell(z, y) = \frac{1}{2}(z-y)^2$ in this experiment and the problem parameters are $n = 500 , d=5000, s=40$ and a sparsity upper bound $\bar{s}=50$ is given to the algorithms instead of the real value. The noise variables $\xi_i$ always follow a Pareto distribution with parameter $\alpha = 2.05$. Apart from that we consider three settings :

\begin{enumerate}[label=(\alph*)]
    \item The gaussian setting : the covariates follow a gaussian distribution.
    \item The heavytailed setting : the covariates are generated from a multivariate Student distribution with $\nu=4.1$ degrees of freedom.
    \item The corrupted setting : the covariates follow the same Student distribution and $5\%$ of the data ($(X_i, Y_i)$ pairs) are corrupted.
\end{enumerate}

We run various algorithms:
\begin{itemize}
    \item AMMD using the trimmed mean estimator (\texttt{AMMD}).
    \item AMDA using the trimmed mean estimator (\texttt{AMDA}).
    \item The iterative thresholding procedure defined in~\cite{liu2019high} using the MOM estimator (\texttt{LLC\_MOM}).
    \item The iterative thresholding procedure defined in~\cite{liu2019high} using the trimmed-mean estimator (\texttt{LLC\_TM}).
    \item Lasso with CGD solver and the trimmed mean estimator as implemented in~\cite{gaiffas2022robust} (\texttt{Lasso\_CGD\_TM}).
    \item Lasso with CGD solver as implemented in Scikit Learn~\cite{scikit-learn} (\texttt{Lasso\_CGD}).
\end{itemize}
Another possible baseline is the algorithm proposed in~\cite{liu2020high}. Nevertheless, we do not include it here because it relies on the outlier removal algorithm inspired from~\cite{balakrishnan2017computationally}. The latter requires to run an SDP subroutine making it excessively slow as soon as the dimension is greater than a few hundreds.

Note that the ``trimmed mean'' estimator used in~\cite{liu2019high} is different from ours since they simply exclude the entries below and above a pair of empirical data quantiles. On the other hand, the estimator we define in Section~\ref{sec:appliVanillaSparse} simply replaces the extreme values by the exceeded threshold before computing an average. This is also called a ``Winsorized mean'' and enjoys better statistical properties.

The algorithms using Lasso~\cite{tibshirani1996regression} optimize an $\ell_1$ regularized objective. The regularization is weighted by a factor $2\sigma \sqrt{\frac{2\log(d)}{n}}$ where $\sigma^2 = \Var(\xi)$ is the noise variance. The previous regularization weight is known to ensure optimal statistical performance, see for instance~\cite{bickel2009simultaneous}.
\begin{figure}[htbp]
     \centering
     \includegraphics[width=\textwidth]{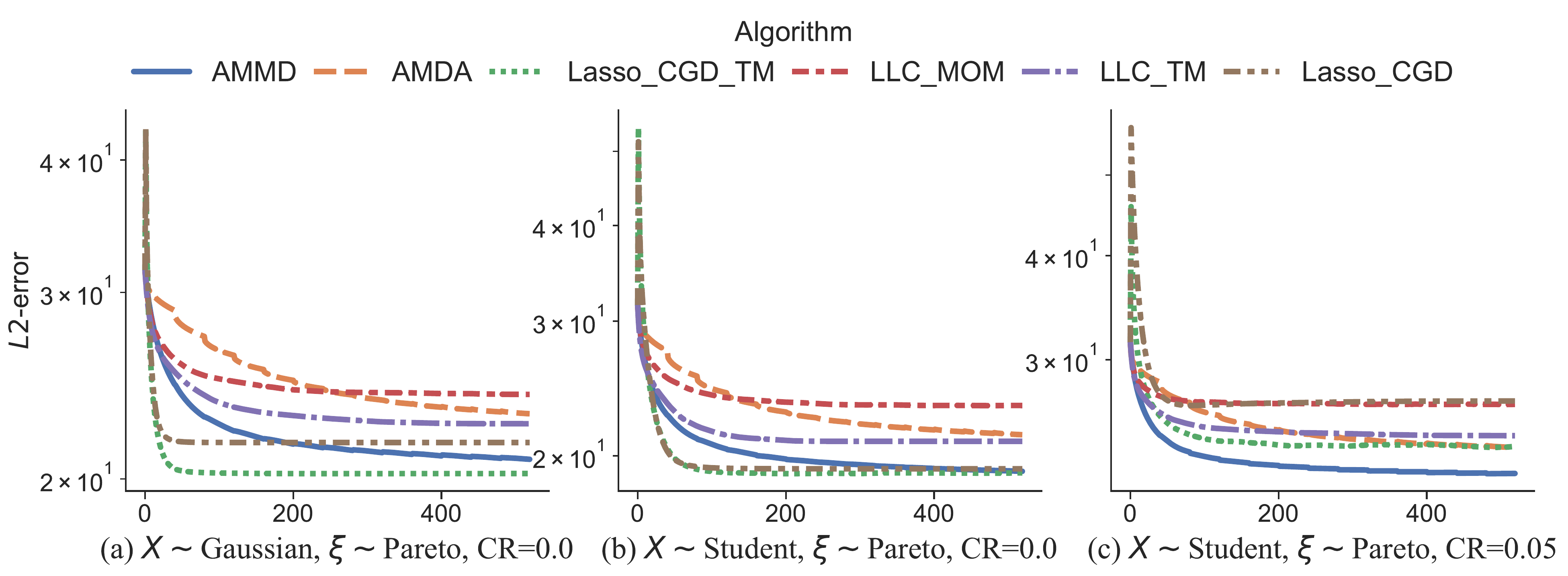}
        \caption{L2 error $\|\theta_t - \theta^\star\|_2$ ($y$-axis) against iterations ($x$-axis) for all the considered algorithms in the simulation settings.}
        \label{fig:lin_reg}
\end{figure}
The experiment is repeated 30 times and the results are averaged. We do not display any confidence intervals for better readability. Figure~\ref{fig:lin_reg} displays the results. We observe that Lasso based methods quickly reach good optima in general and that the version using the robust trimmed mean estimator is sometimes superior in the presence of heavy tails and corruption in particular. \texttt{AMMD} reaches nearly equivalent optima, albeit significantly slower than Lasso methods as seen for settings (a) and (b). However, it is somehow more robust to corruption as seen on setting (c). Unfortunately, the \texttt{AMDA} algorithm struggles to closely approximate the original parameter. We mainly attribute this to slow convergence in settings (a) and (b). Nonetheless, \texttt{AMDA} is among the most robust algorithms to corruption as seen on setting (c). The remaining iterative thresholding based methods \texttt{LLC\_MOM} and \texttt{LLC\_TM} seem to generally stop at suboptimal optima. The Median-Of-Means variant \texttt{LLC\_MOM} is barely more robust than \texttt{Lasso\_CGD} in the corrupted setting (c). The \texttt{LLC\_TM} variant is better but still inferior to \texttt{AMMD}. This reflects the superiority of the (Winsorized) trimmed mean used by \texttt{AMMD} and \texttt{Lasso\_CGD\_TM} to the conventional trimmed mean in \texttt{LLC\_TM}.

\subsection{Sparse classification on real data}

We also carry out experiments on real high dimensional binary classification data sets. These are referred to as \texttt{gina} and \texttt{bioresponse} and were both downloaded from \texttt{openml.org}. We run \texttt{AMMD}, \texttt{AMDA}, \texttt{LLC\_MOM} and \texttt{LLC\_TM} with similar sparsity upperbounds and various levels of corruption and track the objective value, defined using the Logistic loss $\ell(z, y) = \log(1+e^{-zy})$ (with $y\in \mathcal{Y}=\{\pm 1\}$), for each of them. The results are displayed on Figure~\ref{fig:classif_iter} (average over 10 runs). In the non corrupted case, we see that all algorithms reach approximately equivalent optima whereas they display different levels of resilience when corruption is present. In particular, \texttt{LLC\_MOM} is unsurprisingly the most vulnerable since it is based on Median-Of-Means which is not robust to $\eta$-corruption. The rest of the algorithms cope better thanks to the use of trimmed mean estimators, although \texttt{LLC\_TM} seems to be a little less robust which is probably due to the previously mentioned difference in its gradient estimator. Finally, Figure~\ref{fig:classif_iter} also shows that \texttt{AMMD} and \texttt{AMDA} (respectively using Mirror Descent and Dual Averaging) tend to reach generally better final optima despite converging a bit slower than the other algorithms. They also prove to be more stable, even when high step sizes are used.

\begin{figure}[!ht]
    \centering
    \includegraphics[width=\textwidth]{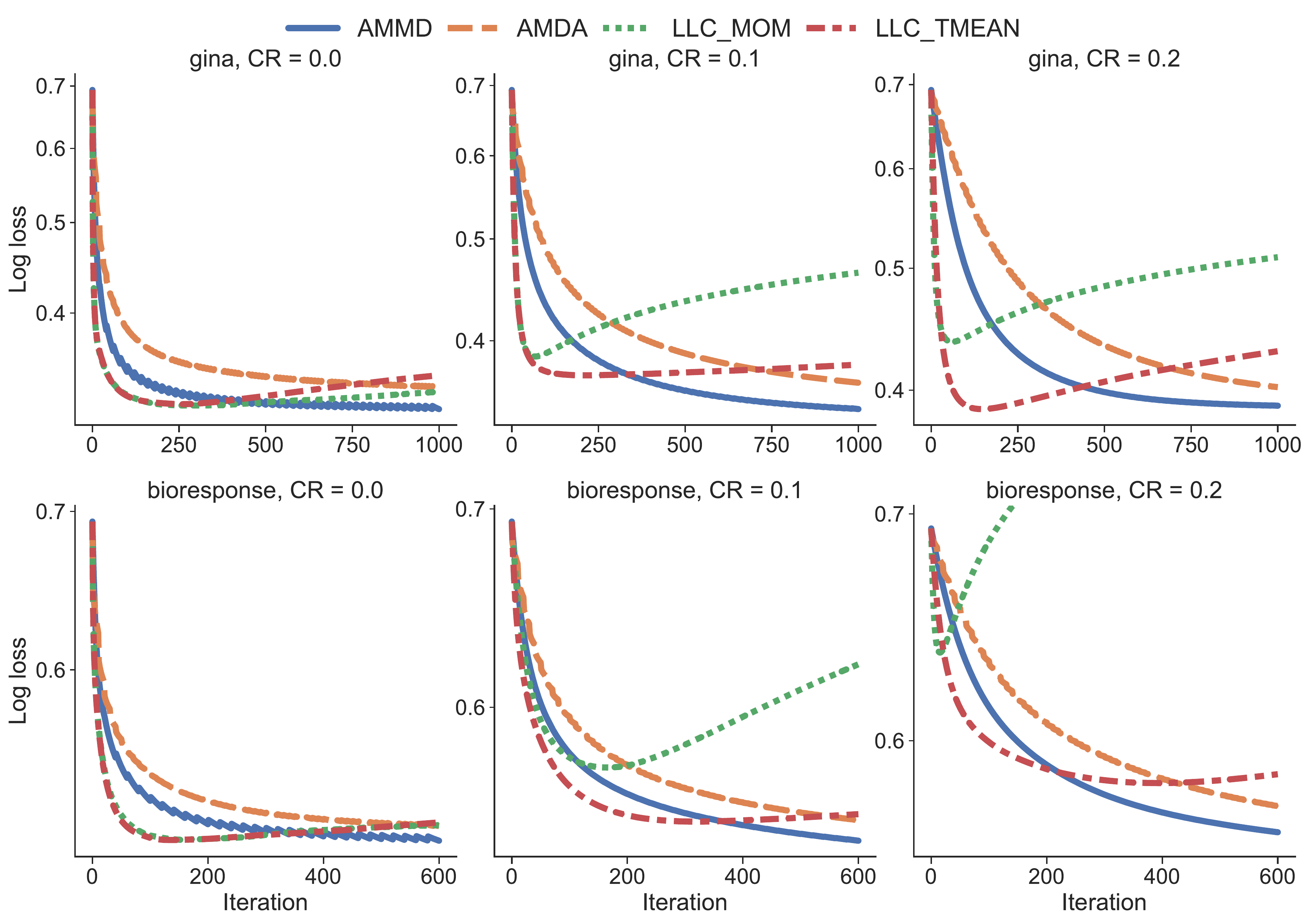}
    \caption{Log loss ($y$-axis) along training iterations ($x$-axis) on two data sets (rows) for $0\%$ corruption (first column), $10\%$ corruption (middle column) and $20\%$ corruption (last column).}
    \label{fig:classif_iter}
\end{figure}

\section{Conclusion}

In this work, we address the problem of robust supervised linear learning in the high-dimensional setting. In order to cover both smooth and non-smooth loss functions, we propose two optimisation algorithms which enjoy linear convergence speeds with only a mild dependence on the dimension. We combine these algorithms with various robust mean estimators, each of them tailored for a specific variant of the sparse estimation problem. We show that the said estimators are robust to heavy-tailed and corrupted data and allow to reach the optimal statistical rates for their respective instances of sparse estimation problems. Furthermore, their computation is efficient which favorably reflects on the computational cost of the overall procedure.
We also confirm our theoretical results through numerical experiments where we evaluate our algorithms in terms of speed, robustness and performance of the final estimates. Finally, we compare our performances with the most relevant concurrent works and discuss the main differences. Perspectives for future work include considering other types of sparsity, divising an algorithm capable of reaching the optimal $s\log(d/s)/n$ rate for vanilla sparsity or considering problems beyond recovery of a single parameter such as, for example, additive sparse and low-rank matrix decomposition.

\acks{This research is supported by the Agence Nationale de la Recherche as part of the ``Investissements d'avenir'' program (reference ANR-19-P3IA-0001; PRAIRIE 3IA Institute).
}

\newpage
\appendix
\section{Proofs}

\subsection{Proof of Lemma~\ref{lem:obj-grad-moment}}\label{proof:obj-grad-moment}

Let $\theta \in \Theta$, using Assumption~\ref{asm:data} we have:
\begin{equation*}
    |\ell(\theta^{\top} X, Y)| \leq C_{\ell, 1} + C_{\ell, 2}|\theta^{\top} X- Y|^{2} \leq C_{\ell, 1} + 2 C_{\ell, 2}(|\theta^{\top} X|^2 + |Y|^2).
\end{equation*}
Taking the expectation and using Assumption~\ref{asm:data} again shows that the objective $\cL(\theta)$ is well defined. Next, for all $j\in\setint{d}$, simple algebra gives:
\begin{align*}
    \big| \ell' (\theta^{\top} X, Y) X_j \big|^2 
    &\leq \big| \big(C_{\ell,1}' + C_{\ell,2}'|\theta^\top X - Y|\big)X^j \big|^2  \\
    &\leq 2   \big(\big|C_{\ell,1}' X^j\big|^ 2  + (C_{\ell,2}' (|(\theta^\top X)  X^j| + |Y  X^j|))^2\big)  \\
    &\leq 2   \Big(\big|C_{\ell,1}' X^j\big|^ 2  + \Big(C_{\ell,2}' \Big(\sum_{k=1}^d|\theta_k|  |(X^k)  X^j| + |Y  X^j|\Big)\Big)^2\Big)  \\
    &\leq 2   \Big(\big|C_{\ell,1}' X^j\big|^ 2 + 2   (C_{\ell,2}')^ 2 \Big(d   \sum_{k=1}^d|\theta_k|^{2} |(X^k)  X^j|^ 2 + |Y  X^j|^2\Big)\Big).
\end{align*}
Recall that we assume $\E|X^j\big|^ 2 < \infty$ and $\E |Y  X^j|^2< \infty$, moreover, using a Cauchy Schwarz inequality, we find:
\begin{align*}
    \E\big|(X^k)  X^j\big|^2\leq \sqrt{ \E \big|X^k\big|^{4} \E \big|X^j\big|^{4}},
\end{align*}
which is also assumed finite. This concludes the proof of Lemma~\ref{lem:obj-grad-moment}.

\subsection{Proofs for Section~\ref{sec:md}}

\subsubsection{Proof of Proposition~\ref{prop:uncorrectedMD}}\label{proof:uncorrectedMD}

We use the abbreviations $\widehat{g}_t = \widehat{g}(\theta_t) $ and $g_t = g(\theta_t)$. Let $\phi \in \Theta$ be any parameter, we first write the optimality condition of the proximal operator defining each step $\theta_{t+1} = \prox_{\beta}(\widehat{g}_t, \theta_t; \theta_0, \Theta)$. Using the convexity and smoothness properties of the objective $\cL$, we find that :
\begin{align}
    \cL(\theta_{t+1}) - \cL(\phi) &= \cL(\theta_{t+1}) - \cL(\theta_{t}) + \cL(\theta_{t}) - \cL(\phi) \nonumber \\
    &\leq \langle g(\theta_t), \theta_{t+1} - \theta_t \rangle + \frac{L}{2}\|\theta_{t+1} - \theta_t\|^2 + \langle g(\theta_t), \theta_{t} - \phi \rangle \nonumber \\
    &= \langle g_t, \theta_{t+1} - \phi \rangle + \frac{L}{2}\|\theta_{t+1} - \theta_t\|^2. \label{eq:decrease}
\end{align}
We have $\widehat{g}_t = g_t + \epsilon_t$ and the optimality condition says that for all $\phi\in \Theta$ we have the inequality : 
\begin{equation*}
    \langle \beta \widehat{g}_t , \phi - \theta_{t+1} \rangle + \langle \nabla \omega_{\theta_0} (\theta_{t+1}) - \nabla \omega_{\theta_0} (\theta_{t}), \phi- \theta_{t+1} \rangle \geq 0.
\end{equation*}
Plugging this into \eqref{eq:decrease} we get :
\begin{align*}
    \cL(\theta_{t+1}) - \cL(\phi) &\leq \frac{1}{\beta}\langle \nabla \omega_{\theta_0} (\theta_{t+1}) - \nabla \omega_{\theta_0} (\theta_{t}), \phi- \theta_{t+1} \rangle + \langle \epsilon_t , \phi - \theta_{t+1} \rangle + \frac{L}{2}\|\theta_{t+1} - \theta_t\|^2 \\
    &= \frac{1}{\beta}(V_{\theta_0}(\phi, \theta_t) - V_{\theta_0}(\theta_{t+1}, \theta_t) - V_{\theta_0}(\phi, \theta_{t+1})) + \langle \epsilon_t , \phi - \theta_{t+1} \rangle + \frac{L}{2}\|\theta_{t+1} - \theta_t\|^2 \\
    &\leq \frac{1}{\beta}(V_{\theta_0}(\phi, \theta_t) - V_{\theta_0}(\phi, \theta_{t+1})) + \langle \epsilon_t , \phi - \theta_{t+1} \rangle,
\end{align*}
where the last step is due to the choice $\beta \leq 1/L$ and the strong convexity of $V$ and the second step follows from the remarkable identity :
\begin{equation*}
    \langle \nabla_{\theta} V_{\theta_0}(\theta, \theta'), z - \theta \rangle = V_{\theta_0}(z, \theta') - V_{\theta_0}(\theta, \theta') - V_{\theta_0}(z, \theta) \quad \text{for all $z,\theta, \theta', \theta_0\in \R^d$}.
\end{equation*}
It suffices to multiply the previous inequality by $\beta$, sum it for $t=0\dots, T-1$ and use the convexity of $\cL$ to find that $\widehat{\theta}_T = \sum_{t=1}^T \theta_t /T$ satisfies :
\begin{equation*}
    \cL(\widehat{\theta}_T) - \cL(\phi) \leq \frac{V_{\theta_0}(\phi, \theta_0) - V_{\theta_0}(\phi, \theta_T)}{\beta T} + \frac{1}{T}\sum_{t=0}^{T-1}\langle \epsilon_t, \phi - \theta_{t+1} \rangle.
\end{equation*}
Then, it only remains to choose $\phi = \theta^\star$ to finish the proof.

\subsubsection{Proof of Proposition~\ref{prop:correctedMD}}\label{proof:correctedMD}

We proceed similarly to the previous Proposition. As previously we have :
\begin{equation*}
    \cL(\theta_{t+1}) - \cL(\phi) \leq \langle g_t, \theta_{t+1} - \phi \rangle + \frac{L}{2}\|\theta_{t+1} - \theta_t\|^2,
\end{equation*}
where $\widehat{g}_t = g_t + \epsilon_t$. Let $\phi \in \Theta$, the optimality condition of $\theta_{t+1} = \widehat{\prox}_{\beta}(\widehat{g}_t, \theta_t; \theta_0, \Theta)$ reads :
\begin{equation*}
    \langle \beta (\widehat{g}_t +\bar{\epsilon}\partial_{\|\cdot\|}(\theta_{t+1} - \theta_t)), \phi - \theta_{t+1} \rangle + \langle \nabla \omega_{\theta_0} (\theta_{t+1}) - \nabla \omega_{\theta_0} (\theta_{t}), \phi- \theta_{t+1} \rangle \geq 0,
\end{equation*}
where $\partial_{\|\cdot\|}(\theta)$ is any subgradient of $\|\cdot\|$ at $\theta$. Plugging this into \eqref{eq:decrease} we get :
\begin{align}
    \cL(\theta_{t+1}) - \cL(\phi) &\leq \frac{1}{\beta}\langle \nabla \omega_{\theta_0} (\theta_{t+1}) - \nabla \omega_{\theta_0} (\theta_{t}), \phi- \theta_{t+1} \rangle + \langle \epsilon_t + \bar{\epsilon}\partial_{\|\cdot\|}(\theta_{t+1} - \theta_t) , \phi - \theta_{t+1} \rangle \nonumber\\
    &\quad + \frac{L}{2}\|\theta_{t+1} - \theta_t\|^2 \nonumber\\
    &= \frac{1}{\beta}(V_{\theta_0}(\phi, \theta_t) - V_{\theta_0}(\theta_{t+1}, \theta_t) - V_{\theta_0}(\phi, \theta_{t+1}))\nonumber\\
    &\quad + \langle \epsilon_t + \bar{\epsilon}\partial_{\|\cdot\|}(\theta_{t+1} - \theta_t) , \phi - \theta_{t+1} \rangle + \frac{L}{2}\|\theta_{t+1} - \theta_t\|^2 \nonumber\\
    &\leq \frac{1}{\beta}(V_{\theta_0}(\phi, \theta_t) - V_{\theta_0}(\phi, \theta_{t+1})) + \langle \epsilon_t + \bar{\epsilon}\partial_{\|\cdot\|}(\theta_{t+1} - \theta_t) , \phi - \theta_{t+1} \rangle, \label{eq:decrease2}
\end{align}
where the last step is due to the choice $\beta \leq 1/L$ and the strong convexity of $V$ and the second step follows from the remarkable identity :
\begin{equation*}
    \langle \nabla_{\theta} V_{\theta_0}(\theta, \theta'), z - \theta \rangle = V_{\theta_0}(z, \theta') - V_{\theta_0}(\theta, \theta') - V_{\theta_0}(z, \theta) \quad \text{for all $z,\theta, \theta', \theta_0\in \R^d$}.
\end{equation*}
Notice that since $\|\epsilon_t\|_{*} \leq \bar{\epsilon}$ for all $t$ and using the identity $\langle \partial_{\|\cdot\|}(\theta), \theta \rangle = \|\theta\|$ which holds for any norm $\|\cdot\|$ we find :
\begin{align*}
    \langle \epsilon_t + \bar{\epsilon}\partial_{\|\cdot\|}(\theta_{t+1} - \theta_t) , \theta_t - \theta_{t+1} \rangle = \langle \epsilon_t, \theta_t - \theta_{t+1} \rangle - \bar{\epsilon}\|\theta_{t+1} - \theta_t\| \leq 0 ,
\end{align*}
so that by taking $\phi = \theta_t$ in \eqref{eq:decrease2} we find that :
\begin{equation*}
    \cL(\theta_{t+1}) \leq \cL(\theta_t),
\end{equation*}
i.e. $\cL(\theta_t)$ is monotonously decreasing with $t$. Using this observation, it suffices to average \eqref{eq:decrease2} over $t=0\dots, T-1$ to find that :
\begin{equation*}
    \cL(\theta_T) - \cL(\phi) \leq \frac{V_{\theta_0}(\phi, \theta_0) - V_{\theta_0}(\phi, \theta_T)}{\beta T} + \frac{1}{T}\sum_{t=0}^{T-1}\langle \epsilon_t + \bar{\epsilon}\partial_{\|\cdot\|}(\theta_{t+1} - \theta_t), \phi - \theta_{t+1} \rangle.
\end{equation*}
From here, the final bound is straight forward to derive by replacing $\phi = \theta^\star$ and using the fact that $\|\partial_{\|\cdot\|}(\theta)\|_* \leq 1$.

\subsubsection{Proof of Theorem~\ref{thm:MDoptimalrate}}\label{proof:thmMD}

The following Lemma is needed for this proof and that of Theorem~\ref{thm:DAoptimalrate}.
\begin{lemma}[{\cite[Lemma A.1]{Juditsky2020SparseRB}}]\label{lem:A1}
    Let $\theta^\star \in \Theta$ be $s$-sparse\textup, $\theta \in \Theta$\textup, and let $\theta_s = \sparse(\theta) \in \argmin \{\|\mu - \theta\| : \mu \in \Theta \:\: s\text{-sparse}\}$. We have \textup:
    \begin{equation*}
        \|\theta_s - \theta^\star\| \leq \sqrt{2s}\|\theta_s - \theta^\star\|_2 \leq 2\sqrt{2s} \|\theta - \theta^\star\|_2.
    \end{equation*}
\end{lemma}

We would like to show by induction that $\|\theta^{(k)} - \theta^\star\| \leq R_{k}$ for $k\geq 0$. In the base case $k=0$, we have $\|\theta^{(0)} - \theta^\star\| \leq R = R_{0}$. For a phase $k+1 \geq 1$ of the approximate Mirror Descent algorithm, assuming the property holds for $k$, by applying Lemma~\ref{lem:A1} and Proposition~\ref{prop:correctedMD} we find :
\begin{equation}\label{eq:prfThmMDobj}
    \cL(\widetilde{\theta}^{(k+1)}) - \cL(\theta^\star) \leq \frac{\nu R_k^2}{\beta T_{k+1}} + 4\bar{\epsilon}R_k \leq 5\bar{\epsilon}R_k,
\end{equation}
where the last inequality uses that that $T_{k+1} = \Big\lceil \frac{\nu R_k }{\beta \bar{\epsilon}}\Big\rceil$. Using the quadratic growth hypothesis (Assumption~\ref{asm:quadgrowth}) leads to :
\begin{equation}\label{eq:prfThmMD}
    \big\|\theta^{(k+1)} - \theta^\star\big\|^2 \leq 2\bar{s}\big\| \theta^{(k+1)} - \theta^\star\big\|_2^2 \leq 8\bar{s} \big\|\widetilde{\theta}^{(k+1)} - \theta^\star\big\|_2^2 \leq \frac{40 \bar{s}\bar{\epsilon}R_k}{\kappa}.
\end{equation}
We have just obtained the bound $\|\theta^{(k+1)} - \theta^\star\big\| =: \widehat{R}_{k+1} \leq h(R_k) := \sqrt{\frac{40 \bar{s}\bar{\epsilon}R_k}{\kappa}}$. It is easy to check that $h(r)$ has a unique fixed point $R^\star := \frac{40 \bar{s}\bar{\epsilon}}{\kappa}$ and that for $r \geq R^\star$ we have $h'(r) \leq 1/2$. Assuming that the former bound holds for $r = R_k$ (otherwise there is nothing to prove) we find :
\begin{align*}
    \widehat{R}_{k+1} - R^\star \leq h(R_k) - h(R^\star) \leq \frac{1}{2}(R_{k} - R^\star) \\
    \implies \widehat{R}_{k+1} \leq \frac{1}{2}(R_{k} + R^\star) = R_{k+1},
\end{align*}
this finishes the induction argument. By unrolling the recursive definition of $R_k$, we obtain that, for all $k \geq 1$ :
\begin{equation*}
    R_k \leq 2^{-k}R_0 + R^\star = 2^{-k}R_0 + \frac{40 \bar{s}\bar{\epsilon}}{\kappa}.
\end{equation*}
The main bound of the Theorem then follows by plugging the above inequality with $k = K-1$ into~\eqref{eq:prfThmMD} and using the fact that $R_0 \geq R^\star$ and the standard inequality $\sqrt{a + b} \leq \sqrt{a} + \sqrt{b}$ which holds for all $a, b \geq 0$. The bound on the objective is obtained similarly.

Let us compute $T$, the total number of iterations necessary for this bound to hold. Given the minimum number of iterations $T_k$ necessary for stage $k$ we have :
\begin{align*}
    T &=\sum_{k=1}^K T_k \leq \sum_{k=1}^K \Big( \frac{\nu R_{k-1}}{\beta \bar{\epsilon}} +1\Big) \leq K + \frac{\nu}{\beta \bar{\epsilon}}\sum_{k=0}^{K-1} (2^{-k}R_0 + R^\star) \\ 
    &\leq \frac{2 R_0 \nu}{\beta \bar{\epsilon}} + K\Big(1 + \frac{40 \nu \bar{s}}{\kappa \beta}\Big).
\end{align*}
This completes the proof.

\subsection{Proofs for Section~\ref{sec:da}}

\subsubsection{Proof of Proposition~\ref{prop:DA}}\label{proof:DA}

For a sequence of iterates $(\theta_t)_{t=0\dots T}$ we introduce the notations :
\begin{equation*}
    \ell_t(\theta) = \sum_{i=0}^t a_i(\cL(\theta_i) + \langle \widehat{g}_i, \theta - \theta_i \rangle) \quad\text{ and }\quad \psi_t^* = \min_{\theta \in \Theta} \ell_t(\theta) + \gamma_t \omega(\theta).
\end{equation*}
We will show the following inequality by induction :
\begin{equation*}
    A_t f(\theta_t) \leq \psi_t^* + \widehat{B}_t,
\end{equation*}
where we define $\widehat{B}_t = \sum_{i=0}^t \frac{a_i^2}{2\gamma_{i-1}} \|g_i\|_*^2 + 2 a_i R\|\epsilon_i\|_*$ with the convention $\gamma_{-1} = \gamma_0$. Assume it holds for $t \geq 0$, since $\gamma_{t+1} \geq \gamma_t$ we have :
\begin{align}
    \psi_{t+1}^* &= \min_{\theta \in \Theta} \ell_t(\theta) + a_{t+1}(\cL(\theta_{t+1}) + \langle \widehat{g}_{t+1} , \theta - \theta_{t+1} \rangle ) + \gamma_{t+1} \omega(\theta) \nonumber \\
    &\geq \min_{\theta \in \Theta} \ell_t(\theta) + a_{t+1}(\cL(\theta_{t+1}) + \langle \widehat{g}_{t+1} , \theta - \theta_{t+1} \rangle ) + \gamma_{t} \omega(\theta). \label{ineq:da1}
\end{align}
Note that, by definition, $\theta_t^+$ realizes the minimum $\psi_t^* = \min_{\theta \in \Theta} \ell_t(\theta) + \gamma_{t} \omega(\theta) = \ell_t(\theta_t^+) + \gamma_{t} \omega(\theta_t^+)$, so that for all $\theta \in \Theta$ we have 
\begin{equation}
    \langle \nabla \ell_t(\theta_t^+) + \gamma_t \nabla \omega(\theta_t^+) , \theta - \theta_t^+ \rangle \geq 0.\label{ineq:da2}
\end{equation} Also, using the convexity of $\ell_t$ and the strong convexity of $\omega(\cdot)$ we have :
\begin{align}
    \ell_t(\theta) + \gamma_t \omega(\theta) \geq &\ell_t(\theta_t^+) + \langle \nabla \ell_t(\theta_t^+), \theta - \theta_t^+ \rangle + \nonumber\\
    &\gamma_t\big(\omega(\theta_t^+) + \langle \nabla \omega(\theta_t^+), \theta - \theta_t^+ \rangle + \frac{1}{2}\|\theta - \theta_t^+\|^2\big).\label{ineq:da3}
\end{align}
By combining Inequalities~\eqref{ineq:da1},~\eqref{ineq:da2} and~\eqref{ineq:da3}, we find that :
\begin{align*}
    \psi_{t+1}^* &\geq \min_{\theta \in \Theta} \psi_t^* + \frac{\gamma_t}{2}\|\theta - \theta_t^+\|^2 + a_{t+1}(\cL(\theta_{t+1}) + \langle \widehat{g}_{t+1} , \theta - \theta_{t+1} \rangle ) .
\end{align*}
Now, using the induction hypothesis $A_t f(\theta_t) \leq \psi_t^* + \widehat{B}_t$, we compute that :
\begin{align*}
    \psi_{t+1}^* &\geq \min_{\theta \in \Theta} A_t f(\theta_t) - \widehat{B}_t + \frac{\gamma_t}{2}\|\theta - \theta_t^+\|^2 + a_{t+1}(\cL(\theta_{t+1}) + \langle \widehat{g}_{t+1} , \theta - \theta_{t+1} \rangle ) \\
    &\geq \min_{\theta \in \Theta} A_t (f(\theta_{t+1}) + \langle g_{t+1}, \theta_t - \theta_{t+1} \rangle) - \widehat{B}_t + \frac{\gamma_t}{2}\|\theta - \theta_t^+\|^2 + a_{t+1}(\cL(\theta_{t+1}) + \\  & \langle g_{t+1} , \theta - \theta_{t+1} \rangle ) - 2Ra_{t+1}\epsilon_{t+1} \\
    &\geq \min_{\theta \in \Theta} A_{t+1} f(\theta_{t+1}) - \widehat{B}_t + \frac{\gamma_t}{2}\|\theta - \theta_t^+\|^2 + a_{t+1}\langle g_{t+1} , \theta - \theta_{t}^+ \rangle  - 2Ra_{t+1}\epsilon_{t+1} \\
    &\geq \min_{\theta \in \Theta} A_{t+1} f(\theta_{t+1}) - \widehat{B}_t - \frac{a_{t+1}^2}{2\gamma_t}\|\widehat{g}_t\|_*^2 - 2Ra_{t+1}\|\epsilon_{t+1}\|_* \\
    &= \min_{\theta \in \Theta} A_{t+1} f(\theta_{t+1}) - \widehat{B}_{t+1},
\end{align*}
where the penultimate inequality uses that $A_{t+1}\theta_{t+1} = A_t \theta_t + a_{t+1} \theta_t^+$. It only remains to check the base case :
\begin{align*}
    \psi_0^* &= \min_{\theta \in \Theta} a_0(\cL(\theta_0) + \langle \widehat{g}_0, \theta - \theta_0 \rangle ) + \frac{\gamma_0}{2}\omega(\theta) \\
    &\geq \min_{\theta \in \Theta} a_0(\cL(\theta_0) + \langle g_0, \theta - \theta_0 \rangle ) + \frac{\gamma_0}{2}\omega(\theta) + \min_{\theta \in \Theta} a_0\langle \epsilon_0, \theta - \theta_0 \rangle\\
    &\geq A_0 \cL(\theta_0) - \frac{a_0^2}{2\gamma_{-1}}\|g_0\|_*^2 -2 a_0 R \| \epsilon_0\|_* = A_0 \cL(\theta_0) - \widehat{B}_0,
\end{align*}
which completes the induction. Now we can compute :
\begin{align*}
    A_t \langle s_t, \theta^\star \rangle + \gamma_t \omega(\theta^\star) &\geq A_t \langle s_t, \theta_t^+ \rangle + \gamma_t \omega(\theta_t^+) +\frac{\gamma_t}{2}\|\theta_t^+ - \theta^\star\|^2 \\
    &\geq \psi_t^* - \sum_{i=0}^t a_i (\cL(\theta_i) - \langle \widehat{g}_i, \theta_i \rangle ) +\frac{\gamma_t}{2}\|\theta_t^+ - \theta^\star\|^2 \\
    &\geq A_t(\cL (\theta_t) -\cL(\theta^\star) + \langle s_t, \theta^\star \rangle) - \widehat{B}_t + \\
    &\sum_{i=0}^t a_i (\cL(\theta^\star) - \cL(\theta_i) - \langle g_i, \theta^\star - \theta_i \rangle ) - \sum_{i=0}^t a_i\langle \epsilon_i,\theta^\star - \theta_i \rangle + \frac{\gamma_t}{2}\|\theta_t^+ - \theta^\star\|^2 ,
\end{align*}
which, by rearranging and using the convexity of $\cL$, leads to :
\begin{align*}
    A_t(\cL(\theta_t) - \cL(\theta^\star)) + \frac{\gamma_t}{2}\|\theta_t^+ - \theta^\star\|^2 &\leq \gamma_t \omega(\theta^\star) + \widehat{B}_t + \sum_{i=0}^t a_i\langle \epsilon_i,\theta^\star - \theta_i \rangle \\
    &\leq \gamma_t \omega(\theta^\star) + \sum_{i=0}^t\frac{a_i^2}{2\gamma_{i-1}}\|g_i\|_*^2 + 4A_t R \bar{\epsilon}.
\end{align*}
For the particular case $a_t = 1$ and $\gamma_t = \sqrt{t+1}$ we have $A_t = t+1$. Moreover, by summing the inequality $\frac{1}{\sqrt{i+1}} \leq \int_{i}^{i+1} \frac{du}{\sqrt{u}}$ we get that $\sum_{i=1}^t\frac{1}{\sqrt{i}} \leq 2\sqrt{t} - 1$. Using this estimate and the Lipschitz property of the objective quickly yields the last part of the Theorem.

\subsubsection{Proof of Theorem~\ref{thm:DAoptimalrate}}\label{proof:thmDA}

We will show by induction that the inequality $\|\theta^{(k)} - \theta^\star\| \leq R_k$ holds for all $k\geq 0$. The case $k=0$ holds by assumption. Note that based on Proposition~\ref{prop:DA} with the choice $a_i = R$ and $\gamma_i = \sqrt{i+1}$ and using the quadratic growth bound~\eqref{eq:dgf_quad_growth} for $\omega$ we get :
\begin{equation}\label{eq:DAappli}
    \cL(\theta_{T'}) - \cL(\theta^\star) \leq \frac{(\nu  + M^2)R}{\sqrt{{T'}}} + 4 R \bar{\epsilon} \leq 5\bar{\epsilon}R,
\end{equation}
where the last step follows from the choice of $T'$. At the same time, at the end of stage $k+1$, we have the alternative :
\begin{itemize}
    \item Either $\|\widetilde{\theta}^{(k+1)} - \theta^\star\|_2 \leq \lambda$ :  then using Lemma~\ref{lem:A1} we find :
    \begin{align}
        \big\|\theta^{(k+1)} - \theta^\star\big\|^2 &= \big\|\sparse_{\bar{s}}(\widetilde{\theta}^{(k+1)}) - \theta^\star\big\|^2 \leq 8\bar{s}\big\|\widetilde{\theta}^{(k+1)} - \theta^\star\big\|_2^2  \nonumber \\
        &\leq \frac{16\lambda \bar{s}\big\|\widetilde{\theta}^{(k+1)} - \theta^\star\big\|_2^2}{\lambda + \big\|\widetilde{\theta}^{(k+1)} - \theta^\star\big\|_2} \leq \frac{16\lambda \bar{s}}{\kappa} \big(\cL(\widetilde{\theta}^{(k+1)}) - \cL(\theta^\star)\big)  \nonumber \\
        &\leq \frac{80\lambda \bar{s}\bar{\epsilon}R_k}{\kappa} = R^\star R_k, \label{eq:DAquad}
    \end{align}
    where the last inequality is an application of~\eqref{eq:DAappli} since $\widetilde{\theta}^{(k+1)} = \theta^{(k+1)}_{T'}$.
    
    \item Or we have $\|\widetilde{\theta}^{(k+1)} - \theta^\star\|_2 > \lambda$ : then only a linear regime holds and using Lemma~\ref{lem:A1} we get :
    \begin{align}
        \big\|\theta^{(k+1)} - \theta^\star\big\| &\leq \sqrt{2\bar{s}}\big\|\theta^{(k+1)} - \theta^\star\big\|_2 \leq 2\sqrt{2\bar{s}}\big\|\widetilde{\theta}^{(k+1)} - \theta^\star\|_2  \nonumber\\
        &\leq \frac{4\sqrt{2 \bar{s}}\big\|\widetilde{\theta}^{(k+1)} - \theta^\star\big\|_2^2/\lambda}{1 + \big\|\widetilde{\theta}^{(k+1)} - \theta^\star\|_2/\lambda} = \frac{4\sqrt{2 \bar{s}}\big\|\widetilde{\theta}^{(k+1)} - \theta^\star\big\|_2^2}{\lambda + \big\|\widetilde{\theta}^{(k+1)} - \theta^\star\big\|_2} \nonumber\\
        &\leq \frac{4\sqrt{2 \bar{s}}}{\kappa} \big(\cL(\widetilde{\theta}^{(k+1)}) - \cL(\theta^\star)\big) \leq \frac{20\bar{\epsilon}R_k\sqrt{2 \bar{s}}}{\kappa} = \tau R_k,\label{eq:DAlin}
    \end{align}
    where we used the inequality $1\leq 2x/(1+x)$ valid for all $x \geq 1$ on the quantity $\|\widetilde{\theta}^{(k+1)} - \theta^\star\|_2/\lambda$. 
        
\end{itemize}
Similarly to the proof of Theorem~\ref{thm:MDoptimalrate}, inequality~\eqref{eq:DAquad} implies that $\big\|\theta^{(k+1)} - \theta^\star\big\| \leq \frac{1}{2}(R_k + R^\star)$ so that we obtained :
\begin{equation*}
    \big\|\theta^{(k+1)} - \theta^\star\big\| \leq \sqrt{2\bar{s}}\big\|\theta^{(k+1)} - \theta^\star\|_2 \leq 2\sqrt{2\bar{s}}\big\|\widetilde{\theta}^{(k+1)} - \theta^\star\|_2 \leq \max\big(\tau R_k, \frac{1}{2}(R_k + R^\star)\big) = R_{k+1},
\end{equation*}
which finishes the induction to show~\eqref{eq:thmDAparamIneq}. Inequality~\eqref{eq:thmDAobjIneq} then follows using~\eqref{eq:DAappli}. Note that, since we assume $\tau < 1$ and $R_0 \geq R^\star$, the sequence $(R_k)_{k\geq 0}$ is decreasing. We now distinguish two phases :
\begin{itemize}
    \item The linear phase : if $\tau R_0 > \frac{1}{2}(R_0 + R^\star)$ (which implies $\tau > 1/2$ since $R_0 \geq R^\star$) then while $\tau R_k > \frac{1}{2}(R_k + R^\star)$ we have $R_{k+1} = \tau R_k = \tau^{k+1}R_0$ and the number of stages necessary to reverse the previous inequality is :
    \begin{equation*}
        \log\Big(\frac{(2\tau - 1)R_0}{R^\star}\Big)\Big/\log(1/\tau) \leq \log\Big(\frac{R_0}{R^\star}\Big)\Big/\log(1/\tau).
    \end{equation*}
    \item The quadratic phase : let $K_1 \geq 0$ be the first stage index such that we have $\tau R_k \leq \frac{1}{2}(R_k + R^\star)$ and so $R_{k+1} = \frac{1}{2}(R_k + R^\star)$ and by iteration $R_{l + K_1} \leq 2^{-l}R_{K_1} + R^\star.$ In all cases $R_{K_1} \leq R_0$ so the number of necessary stages is :
    \begin{equation*}
        \frac{\log(R_{K_1}/R^\star)}{\log(2)} \leq \frac{\log(R_0/R^\star)}{\log(2)}.
    \end{equation*}
\end{itemize}
We have shown that the overall number of necessary stages is at most $\log(R_0/R^\star)\Big(\frac{1}{\log(2)} + \frac{1}{\log(1/\tau)}\Big)$. The Theorem's final claim then follows since the number of per-stage iterations is constant equal to $T' = \Big\lceil \Big(\frac{\nu + M^2}{\bar{\epsilon}}\Big)^2 \Big\rceil$.

\subsubsection{Dual averaging for vanilla sparse estimation}\label{sec:DAforVanilla}
\begin{corollary}
In the context of Theorem~\ref{thm:DAoptimalrate} and Lemma~\ref{lem:coordTM}, let the AMDA algorithm be run starting from $\theta_0\in\Theta = B_{\|\cdot\|}(\theta_0, R)$ using the coordinatewise trimmed mean estimator with sample splitting i.e. at each iteration a different batch of size $\widetilde{n} = n/T$ is used for gradient estimation with confidence $\widetilde{\delta} = \delta/T$ where $T$ is the total number of iterations. Let $K$ be the number of stages and $\widehat{\theta}$ the obtained estimator. Denote $\sigma_{\max}^2 = \sup_{\theta \in \Theta} \max_{j\in\setint{d}}\Var (\ell'(\theta^\top X, Y) X^j)$, with probability at least $1 - \delta$, the latter satisfies :
    \begin{equation*}
        \big\|\widehat{\theta} - \theta^\star\big\|_2 \leq \tau^{K\wedge K_1}2^{(K_1\!-\!K)\wedge 0}\frac{ R}{\sqrt{2\bar{s}}} + \frac{280\lambda\sqrt{2\bar{s}}\sigma_{\max}}{\kappa}\sqrt{4\eta + 6 \frac{\log(4/\widetilde{\delta}) + \log(d)}{\widetilde{n}}}.
    \end{equation*}
    with $K_1$ an integer such that $K_1 \le \log\Big(\frac{\kappa R}{80\lambda \bar{s}\bar{\epsilon}}\Big)\Big/\log(1/\tau)$ with $\tau = \frac{10\sqrt{8\bar{s}} \bar{\epsilon}}{\kappa} < 1$ by assumption and
    \begin{equation*}
        \bar{\epsilon} = 7 \sigma_{\max} \sqrt{4\eta + 6\frac{\log(4/\delta) + \log(d)}{n}}.
    \end{equation*}
\end{corollary}
\begin{proof}
By the proof of Theorem~\ref{thm:DAoptimalrate}, we have $R=R_0$ and $K_1$ is defined as the first stage index $k$ such that we have $\tau R_k \leq \frac{1}{2}(R_k + R^\star)$ implying that $R_{k+1} = \frac{1}{2}(R_k + R^\star)$ and hence $R_{l + K_1} \leq 2^{-l}R_{K_1} + R^\star$ for $l\geq 0.$ We also had $K_1 \le \log\Big(\frac{R}{R^\star}\Big)\Big/\log(1/\tau)$ with $R^\star = \frac{80\lambda \bar{s}\bar{\epsilon}}{\kappa}.$ Using Theorem~\ref{thm:DAoptimalrate}, it follows that:
\begin{equation*}
    \sqrt{2\bar{s}}\big\|\theta^{(k)} - \theta^{\star}\big\|_2 \leq R_k \leq 2^{(K_1 - k)\wedge 0} \tau^{k\wedge K_1} R + R^{\star},
\end{equation*}
whence the result is easily obtained by plugging the value of $R^\star$ and using Lemma~\ref{lem:coordTM} with a union bound argument over all iterations $T$ in order to bound $\bar{\epsilon} = \max_i \|\epsilon_i\|_*$ as defined in Proposition~\ref{prop:DA}.
\end{proof}

\subsection{Closed form computation of the prox operator}\label{apd:closed-form-prox}

In both Sections~\ref{sec:md} and~\ref{sec:da} the optimization methods are defined using a prox operator which involves solving an optimization problem of the following form:
\begin{equation*}
    \argmin_{\|\theta\| \leq R} \langle u, \theta \rangle + \omega(\theta)
\end{equation*}
for some $R > 0$ and $u \in \Theta^*$.

\subsubsection{Vanilla/Group-sparse case}

We consider the group sparse case where $\Theta \subset \R^{d \times K}$, the groups are the rows of $\theta \in \Theta$ and we use the norm $\|\cdot\| = \|\cdot\|_{1,2}$ and the usual scalar product $\langle u, \theta \rangle = u^\top \theta$. We use the prox function $\omega$ defined as $\omega(\theta) = C \|\theta\|_{p,2}^2 = C \big(\sum_{i=1}^d \|\theta_{i,:}\|_2^p\big)^{2/p}$ with $p = 1 + 1/\log(d)$ and $\theta_{i,:}$ the $i$-th row of $\theta$. Note that, for $K=1$ we retrieve the usual linear learning setting. The setting $K > 1$ can be used, for example, for multiclass classification.

In order to obtain a closed form solution, we start by writing the Lagrangian :
\begin{equation}\label{eq:proxproblem}
    \cL (\theta) = \langle u, \theta \rangle + \omega(\theta) + \lambda (\|\theta\| - R),
\end{equation}
where we introduce the multiplier $\lambda \geq 0$. We initially assume the latter known and try to find a critical point for $\cL$ that is $\theta \in \R^{d\times K}$ such that :
\begin{equation*}
    \partial\cL (\theta) = u + \nabla \omega(\theta) + \partial \|\cdot\|_{p,2}(\theta) \ni 0,
\end{equation*}
where we denoted $\partial \|\cdot\|_{1,2}(\theta)$ the subdifferential of the norm $\|\cdot\|_{1,2}$ since the latter is not differentiable whenever $\theta_{i,:} = 0$ for some $i \in \setint{d}$.

Defining the function $h_{\alpha}(\theta) : \R^{d\times K} \to \R^{d\times K}$ such that $h_{\alpha}(\theta)_{i,j} = \frac{\theta_{i,j}}{\|\theta_{i,:}\|_2^{\alpha}}$ for $\theta$ such that $\theta_{i,:} \neq 0$ for all $i$, we can write for such $\theta$ :
\begin{equation*}
    \nabla \omega(\theta) = 2C\|\theta\|_{p,2}^{2-p} h_{2-p}(\theta) \quad \text{ and }\quad \partial \|\cdot\|_{1,2}(\theta) = h_1(\theta).
\end{equation*}
When $\theta_{i,:} = 0$ for some $i$, a subgradient of $\|\cdot\|_{1,2}$ can be obtained by using this definition and plugging any subunit vector for index $i$. Assuming that $\theta_{i,:} \neq 0$ for all $i\in\setint{d}$, a critical point of the Lagrangian must satisfy for all $i,j$ :
\begin{equation*}
    u_{i,j} + \theta_{i,j} \Big( 2C \Big( \frac{\|\theta\|_{p,2}}{\|\theta_{i,:}\|_2} \Big)^{2-p} + \frac{\lambda}{\|\theta_{i,:}\|_2} \Big) = 0.
\end{equation*}
From here, a quick computation yields for all $i$ that :
\begin{equation*}
    \|\theta_{i,:}\|_2 = \frac{\|u_{i,:}\|_2}{  \Big( 2C \Big( \frac{\|\theta\|_{p,2}}{\|\theta_{i,:}\|_2} \Big)^{2-p} + \frac{\lambda}{\|\theta_{i,:}\|_2} \Big)} \quad \text{and hence} \quad 2C \|\theta\|_{p,2}^{2-p} \|\theta_{i,:}\|_2^{p-1}  = \|u_{i,:}\|_2 - \lambda. 
\end{equation*}
Notice that this equality cannot hold when $\|u_{i,:}\|_2 < \lambda$, in this case, we deduce that $\theta_{i,:} = 0$ which satisfies the critical point condition. This leads to the relation :
\begin{equation*}
    \theta_{i,j} = \frac{- \beta_i u_{i,j}}{ 2C \Big( \frac{\|\theta\|_{p,2}}{\|\theta_{i,:}\|_2} \Big)^{2-p} + \frac{\lambda}{\|\theta_{i,:}\|_2} }\quad \text{ with $\beta_i = \ind{\|u_{i,:}\|_2 > \lambda}$} .
\end{equation*}
From here, it is easy to figure out that $\|\theta\|_{p,2} = \Big( \sum_{i=1}^d \beta_i \Big( \frac{\|u_{i,:}\|_2 - \lambda}{2C} \Big)^{p/(p-1)} \Big)^{(p-1)/p}$. All computations are now possible knowing $\lambda$.

To find the latter's value, we plug the formula we have  for $\theta$ into the constraint $\|\theta\|_{1,2} \leq R$. After a few manipulations, we find the constraint is satisfied for $\lambda$ such that :
\begin{equation*}
    \frac{1}{2C} \Big( \sum_{i=1}^d \beta_i \big( \|u_{i,:}\|_2 - \lambda \big)^{1/(p-1)} \Big) \Big( \sum_{i=1}^d \beta_i \big( \|u_{i,:}\|_2 - \lambda \big)^{p/(p-1)} \Big)^{(p-2)/p} \leq R,
\end{equation*}
(recall that the $\beta_i$s also depend on $\lambda$). It only remains to choose the smallest $\lambda \geq 0$ such that the above inequality holds. 
\subsubsection{Low-rank matrix case}

In the low-rank matrix case the parameter space is $\Theta \subset \R^{p\times q}$ and the norm $\|\cdot\|$ represents the nuclear norm $\|\theta\| = \|\sigma(\theta)\|_1$ with $\sigma(\theta).$ The d.-g.f./proximal function is the scaled squared $p$-Schatten norm $\omega(\theta) = C\|\sigma(\theta)\|_p^2$ with $p = 1 + \frac{1}{12\log(q)} = 1+r$ and the scalar product is defined as $\langle u, v \rangle = \Tr(u^\top v)$.

We introduce the notation $h_r(\theta) = \|\sigma(\theta)\|_{1+r}$. Letting $UDV^\top$ denote an SVD of $\theta \in \R^{p\times q}$, the gradient of $h_r$ at $\theta$ for $r > 0$ is given by $\nabla h_r(\theta) = U\Big( \frac{D}{h_r(\theta)} \Big)^r V^\top$. The nuclear norm $h_0$ is not differentiable but a subgradient is given by $\partial h_0 (\theta) = UV^\top + W$ for any $W$ such that $\|W\|_{\op} \leq 1$ and $ U U^\top W = 0$ and $W V V^\top = 0$ (see ~\cite{watson1992characterization}).

In order to define the $\prox$ operator in this setting we need to solve problem~\eqref{eq:proxproblem} again which amounts to finding a critical point $\theta = UDV^\top$ such that :
\begin{equation*}
    u + 2C h_r(\theta)^{1-r}U D^r V^\top  + \lambda U V^\top \ni 0.
\end{equation*}
We define $\theta$ by choosing $U$ and $V$ such that $u = U D_u V^\top$ is an SVD of $u$. Thanks to this choice, it only remains to choose $D = \diag(\sigma(\theta))$ properly in order to have :
\begin{equation*}
    \sigma(u) + 2C \|\sigma(\theta)\|_{1+r}^{1-r}\sigma(\theta)^r + \lambda (\ind{\sigma(\theta) \neq 0} + w) = 0,
\end{equation*}
where the power $\sigma(\theta)^r$ is computed coordinatewise, $\ind{\sigma(\theta) \neq 0}$ is the indicator vector of non zero coordinates of $\sigma(\theta)$ and for some vector $w$ such that $|w_j| \leq 1$ for all $j$ supported on the coordinates where $\sigma(\theta) = 0$.

The problem then becomes analogous to finding the proximal operator for Vanilla sparsity and after some computations we find that the solution is given by :
\begin{equation*}
    \sigma(\theta) = -\frac{\softhresh_\lambda (u)^{1/r}}{ 2C \|\softhresh_\lambda (u)\|_{(1+r)/r}^{(1-r)/r}},
\end{equation*}
where the soft threshold operator is defined by $\softhresh_\lambda(u)_j = \sign(u_j)\max(0, |u_j| - \lambda)$ and $\lambda$ is the smallest real number such that :
\begin{equation*}
    \frac{1}{2C} \|\softhresh_\lambda(u)\|_{1/r}^{1/r} \cdot \|\softhresh_\lambda(u)\|_{(r-1)/r}^{(r+1)/r} \leq R.
\end{equation*}

\subsection{Proof of Proposition~\ref{prop:spectralMOM}}
\label{proof:spectralMOM}

Let $B_1, \dots, B_K$ be a partition of $\setint{n}$ into disjoint equal sized blocks. We assume that the number of outliers is $|\cO| \leq (1 - \varepsilon)K/2$ where $0 < \varepsilon < 1$ will be fixed later. Let $\cK = \{k\in \setint{K} \: : \: B_k \cap \cO = \emptyset\}$ be the set of outlier-free blocks. Denote the block means for $j \in \setint{K}$ as 
\begin{equation*}
    \xi^{(j)} = \frac{1}{m\chi }\sum_{i \in B_j} \psi(\chi \widetilde{A}_i),
\end{equation*} 
where $\chi$ is temporarily chosen as $\chi = \sqrt{\frac{2m\log(2(p+q)/\delta')}{v(A)}}$ for some $0 < \delta'<1/2$. By applying Corollary 3.1 from~\cite{minsker2018sub}, we obtain that with probability at least $1- \delta'$ 
\begin{equation}
    \big\|\widehat{\mu}^{(k)} - \mu\big\|_{\op} \leq \sqrt{\frac{2v(A)\log(2(p+q)/\delta'))}{m}}.\label{ineq:minsker}
\end{equation} 
Now let $r_{jl} = \|\widehat{\mu}^{(j)} - \widehat{\mu}^{(l)}\|_{\op}$, denote $r^{(j)}$ the increasingly sorted version of $r_{j:}$ and let $\widehat{j} \in \argmin_j r^{(j)}_{\lceil K/2 \rceil}$ so that $\widehat{\mu} = \widehat{\mu}^{(\widehat{j})}$.

Define the events $E_j = \Big\{\big\|\widehat{\mu}^{(j)} - \mu\big\|_{\op} \leq \sqrt{\frac{2v(A)\log(2(p+q)/\delta'))}{m}}\Big\}$ and assume that we have $\sum_{k=1}^K \ind{E_k} > K/2$ i.e. over half of the block means satisfy Inequality~\eqref{ineq:minsker} simultaneously. Then there exists $j' \in \setint{K}$  such that the block $\widehat{j}$ satisfies
\begin{equation*}
    r^{(\widehat{j})}_{K/2}  = \big\|\widehat{\mu}^{(\widehat{j})} - \widehat{\mu}^{(j')}\big\|_{\op} \leq \big\|\widehat{\mu}^{(\widehat{j})} - \mu\big\|_{\op} + \big\|\widehat{\mu}^{(j')} - \mu\big\|_{\op} \leq 2 \sqrt{\frac{2v(A)\log(2(p+q)/\delta'))}{m}}.
\end{equation*}
Moreover, among the $K/2$ block means closest to $\widehat{\mu}^{(\widehat{j})}$, at least one of them $\widehat{\mu}^{(j'')}$ satisfies $\|\widehat{\mu}^{(j'')} - \mu\|_{\op} \leq \sqrt{\frac{2v(A)\log(2(p+q)/\delta'))}{m}}$ thus we find :
\begin{align*}
    \big\|\widehat{\mu} - \mu\big\|_{\op} = \big\|\widehat{\mu}^{(\widehat{j})} - \mu\big\|_{\op} &\leq \big\|\widehat{\mu}^{(\widehat{j})} - \widehat{\mu}^{(j'')}\big\|_{\op} + \big\|\widehat{\mu}^{(j'')} - \mu\big\|_{\op} \\
    &\leq r^{(\widehat{j})}_{K/2} \!+\! \sqrt{\frac{2v(A)\log(2(p\!+\!q)/\delta'))}{m}} \leq 3 \sqrt{\frac{2v(A)\log(2(p\!+\!q)/\delta'))}{m}}.
\end{align*}
Finally, let us show that $\sum_{k=1}^K \ind{E_k} > K/2$ happens with high probability. Observe that for $j \in \mathcal{K}$ the variables $\ind{\overline{E}_j}$ are stochastically dominated by Bernoulli variables with parameter $\delta'$ so that their sum is stochastically dominated by a Binomial random variable $S := \binomial (|\cK|, \delta')$. We compute :
\begin{align*}
    \Proba\Big(\sum_{k=1}^K \ind{E_k} < K/2\Big) &= \Proba\Big(\sum_{k=1}^K \ind{\overline{E}_k} > K/2\Big) \\ 
    &\leq \Proba\Big(|\cO| + \sum_{k\in \cK} \ind{\overline{E}_k} > K/2\Big) \\ 
    &\leq \Proba(S - \E S > K/2 - |\cO| - \delta' |\cK|) \\ 
    &\leq \Proba(S - \E S > K(\varepsilon - 2 \delta')/2) \\ 
    &\leq \exp (-K(\varepsilon - 2 \delta')^2/2)
\end{align*}
where we used that $|\cO| \leq (1 - \varepsilon)K/2$, $|\cK| \leq K$ and Hoeffding's inequality at the end. Choosing $\varepsilon = 5/6$, $\delta' = 1/4$ and recalling the choice of $K$ and that $m = n/K$ we finish the proof.

\bibliographystyle{plain}%apalike}%
\bibliography{ref}

\begin{thebibliography}{10}

\bibitem{agarwal2012stochastic}
Alekh Agarwal, Sahand Negahban, and Martin~J Wainwright.
\newblock Stochastic optimization and sparse statistical recovery: Optimal
  algorithms for high dimensions.
\newblock {\em Advances in Neural Information Processing Systems}, 25, 2012.

\bibitem{alon1999space}
Noga Alon, Yossi Matias, and Mario Szegedy.
\newblock The space complexity of approximating the frequency moments.
\newblock {\em Journal of Computer and system sciences}, 58(1):137--147, 1999.

\bibitem{audibert2011robust}
Jean-Yves Audibert and Olivier Catoni.
\newblock Robust linear least squares regression.
\newblock {\em The Annals of Statistics}, 39(5):2766--2794, 2011.

\bibitem{bakshi2021robust}
Ainesh Bakshi and Adarsh Prasad.
\newblock Robust linear regression: Optimal rates in polynomial time.
\newblock In {\em Proceedings of the 53rd Annual ACM SIGACT Symposium on Theory
  of Computing}, pages 102--115, 2021.

\bibitem{balakrishnan2017computationally}
Sivaraman Balakrishnan, Simon~S Du, Jerry Li, and Aarti Singh.
\newblock Computationally efficient robust sparse estimation in high
  dimensions.
\newblock In {\em Conference on Learning Theory}, pages 169--212. PMLR, 2017.

\bibitem{bellec2018slope}
Pierre~C Bellec, Guillaume Lecu{\'e}, and Alexandre~B Tsybakov.
\newblock Slope meets lasso: improved oracle bounds and optimality.
\newblock {\em The Annals of Statistics}, 46(6B):3603--3642, 2018.

\bibitem{bickel2009simultaneous}
Peter~J Bickel, Ya’acov Ritov, and Alexandre~B Tsybakov.
\newblock Simultaneous analysis of lasso and dantzig selector.
\newblock {\em The Annals of Statistics}, 37(4):1705--1732, 2009.

\bibitem{blumensath2009iterative}
Thomas Blumensath and Mike~E Davies.
\newblock Iterative hard thresholding for compressed sensing.
\newblock {\em Applied and Computational Harmonic Analysis}, 27(3):265--274,
  2009.

\bibitem{blumensath2010normalized}
Thomas Blumensath and Mike~E Davies.
\newblock Normalized iterative hard thresholding: Guaranteed stability and
  performance.
\newblock {\em IEEE Journal of Selected Topics in Signal Processing},
  4(2):298--309, 2010.

\bibitem{bogdan2015slope}
Ma{\l}gorzata Bogdan, Ewout Van Den~Berg, Chiara Sabatti, Weijie Su, and
  Emmanuel~J Cand{\`e}s.
\newblock Slope—adaptive variable selection via convex optimization.
\newblock {\em The Annals of Applied Statistics}, 9(3):1103, 2015.

\bibitem{bondell2008simultaneous}
Howard~D Bondell and Brian~J Reich.
\newblock Simultaneous regression shrinkage, variable selection, and supervised
  clustering of predictors with oscar.
\newblock {\em Biometrics}, 64(1):115--123, 2008.

\bibitem{buhlmann2011statistics}
Peter B{\"u}hlmann and Sara Van De~Geer.
\newblock {\em Statistics for High-Dimensional Data: Methods, Theory and
  Applications}.
\newblock Springer Science \& Business Media, 2011.

\bibitem{bunea2007sparsity}
Florentina Bunea, Alexandre Tsybakov, and Marten Wegkamp.
\newblock Sparsity oracle inequalities for the lasso.
\newblock {\em Electronic Journal of Statistics}, 1:169--194, 2007.

\bibitem{candes2011tight}
Emmanuel~J Cand{\`e}s and Yaniv Plan.
\newblock Tight oracle inequalities for low-rank matrix recovery from a minimal
  number of noisy random measurements.
\newblock {\em IEEE Transactions on Information Theory}, 57(4):2342--2359,
  2011.

\bibitem{candes2009exact}
Emmanuel~J Cand{\`e}s and Benjamin Recht.
\newblock Exact matrix completion via convex optimization.
\newblock {\em Foundations of Computational mathematics}, 9(6):717--772, 2009.

\bibitem{castillo2015bayesian}
Isma{\"e}l Castillo, Johannes Schmidt-Hieber, and Aad van~der Vaart.
\newblock {Bayesian linear regression with sparse priors}.
\newblock {\em The Annals of Statistics}, 43(5):1986--2018, 2015.

\bibitem{catoni2012challenging}
Olivier Catoni.
\newblock Challenging the empirical mean and empirical variance: a deviation
  study.
\newblock In {\em Annales de l'Institut Henri Poincar{\'e}, Probabilit{\'e}s et
  Statistiques}, volume~48, pages 1148--1185. Institut Henri Poincar{\'e},
  2012.

\bibitem{pmlr-v28-chen13h}
Yudong Chen, Constantine Caramanis, and Shie Mannor.
\newblock Robust sparse regression under adversarial corruption.
\newblock In {\em Proceedings of the 30th International Conference on Machine
  Learning}, pages 774--782. PMLR, 2013.

\bibitem{cherapanamjeri2020optimal}
Yeshwanth Cherapanamjeri, Efe Aras, Nilesh Tripuraneni, Michael~I Jordan,
  Nicolas Flammarion, and Peter~L Bartlett.
\newblock Optimal robust linear regression in nearly linear time.
\newblock {\em arXiv preprint arXiv:2007.08137}, 2020.

\bibitem{pmlr-v99-cherapanamjeri19b}
Yeshwanth Cherapanamjeri, Nicolas Flammarion, and Peter~L. Bartlett.
\newblock Fast mean estimation with sub-gaussian rates.
\newblock In {\em Conference on Learning Theory}, pages 786--806. PMLR, 2019.

\bibitem{cormen2009introduction}
Thomas~H. Cormen, Charles~E. Leiserson, Ronald~L. Rivest, and Clifford Stein.
\newblock {\em Introduction to algorithms}.
\newblock MIT press, 2009.

\bibitem{dalalyan2019outlier}
Arnak Dalalyan and Philip Thompson.
\newblock {Outlier-robust estimation of a sparse linear model using $\ell_1
  $-penalized Huber's $ M $-estimator}.
\newblock {\em Advances in Neural Information Processing Systems}, 32, 2019.

\bibitem{d2004direct}
Alexandre d'Aspremont, Laurent Ghaoui, Michael Jordan, and Gert Lanckriet.
\newblock A direct formulation for sparse pca using semidefinite programming.
\newblock {\em Advances in Neural Information Processing Systems}, 17, 2004.

\bibitem{Depersin2019RobustSE}
Jules Depersin and Guillaume Lecu{\'e}.
\newblock Robust subgaussian estimation of a mean vector in nearly linear time.
\newblock {\em arXiv preprint arXiv:1906.03058}, 2019.

\bibitem{diakonikolas2019robust}
Ilias Diakonikolas, Gautam Kamath, Daniel Kane, Jerry Li, Ankur Moitra, and
  Alistair Stewart.
\newblock Robust estimators in high-dimensions without the computational
  intractability.
\newblock {\em SIAM Journal on Computing}, 48(2):742--864, 2019.

\bibitem{diakonikolas2020outlier}
Ilias Diakonikolas, Daniel~M Kane, and Ankit Pensia.
\newblock Outlier robust mean estimation with subgaussian rates via stability.
\newblock {\em Advances in Neural Information Processing Systems},
  33:1830--1840, 2020.

\bibitem{diakonikolas2017statistical}
Ilias Diakonikolas, Daniel~M Kane, and Alistair Stewart.
\newblock Statistical query lower bounds for robust estimation of
  high-dimensional gaussians and gaussian mixtures.
\newblock In {\em 2017 IEEE 58th Annual Symposium on Foundations of Computer
  Science (FOCS)}, pages 73--84. IEEE, 2017.

\bibitem{donoho2000high}
David~L Donoho et~al.
\newblock High-dimensional data analysis: The curses and blessings of
  dimensionality.
\newblock {\em AMS Math Challenges Lecture}, 1(2000):32, 2000.

\bibitem{duchi2010composite}
John~C Duchi, Shai Shalev-Shwartz, Yoram Singer, and Ambuj Tewari.
\newblock Composite objective mirror descent.
\newblock In {\em Conference on Learning Theory}, volume~10, pages 14--26.
  Citeseer, 2010.

\bibitem{fan2021shrinkage}
Jianqing Fan, Weichen Wang, and Ziwei Zhu.
\newblock A shrinkage principle for heavy-tailed data: High-dimensional robust
  low-rank matrix recovery.
\newblock {\em The Annals of statistics}, 49(3):1239, 2021.

\bibitem{filzmoser2021robust}
Peter Filzmoser and Klaus Nordhausen.
\newblock Robust linear regression for high-dimensional data: An overview.
\newblock {\em Wiley Interdisciplinary Reviews: Computational Statistics},
  13(4):e1524, 2021.

\bibitem{gaiffas2022robust}
St{\'e}phane Ga{\"\i}ffas and Ibrahim Merad.
\newblock Robust supervised learning with coordinate gradient descent.
\newblock {\em arXiv preprint arXiv:2201.13372}, 2022.

\bibitem{hampel1971}
Frank~R Hampel.
\newblock {A General Qualitative Definition of Robustness}.
\newblock {\em The Annals of Mathematical Statistics}, 42(6):1887 -- 1896,
  1971.

\bibitem{hastie2015statistical}
Trevor Hastie, Robert Tibshirani, and Martin Wainwright.
\newblock Statistical learning with sparsity.
\newblock {\em Monographs on Statistics and Applied Probability}, 143:143,
  2015.

\bibitem{pmlr-v97-holland19a}
Matthew Holland and Kazushi Ikeda.
\newblock Better generalization with less data using robust gradient descent.
\newblock In {\em International Conference on Machine Learning}, pages
  2761--2770. PMLR, 2019.

\bibitem{Hopkins2018MeanEW}
Samuel~B Hopkins.
\newblock Mean estimation with sub-{G}aussian rates in polynomial time.
\newblock {\em arXiv: Statistics Theory}, 2018.

\bibitem{hsu2016loss}
Daniel Hsu and Sivan Sabato.
\newblock Loss minimization and parameter estimation with heavy tails.
\newblock {\em The Journal of Machine Learning Research}, 17(1):543--582, 2016.

\bibitem{huang2010variable}
Jian Huang, Joel~L Horowitz, and Fengrong Wei.
\newblock Variable selection in nonparametric additive models.
\newblock {\em The Annals of Statistics}, 38(4):2282, 2010.

\bibitem{huber1964robust}
Peter~J Huber.
\newblock Robust estimation of a location parameter.
\newblock {\em The Annals of Mathematical Statistics}, 35(1):73--101, 1964.

\bibitem{huber2004robust}
Peter~J Huber.
\newblock {\em Robust statistics}, volume 523.
\newblock John Wiley \& Sons, 2004.

\bibitem{jain2016structured}
Prateek Jain, Nikhil Rao, and Inderjit~S Dhillon.
\newblock Structured sparse regression via greedy hard thresholding.
\newblock {\em Advances in Neural Information Processing Systems}, 29, 2016.

\bibitem{jain2014iterative}
Prateek Jain, Ambuj Tewari, and Purushottam Kar.
\newblock On iterative hard thresholding methods for high-dimensional
  $m$-estimation.
\newblock {\em Advances in Neural Information Processing Systems}, 27, 2014.

\bibitem{JERRUM1986169}
Mark~R Jerrum, Leslie~G Valiant, and Vijay~V Vazirani.
\newblock Random generation of combinatorial structures from a uniform
  distribution.
\newblock {\em Theoretical Computer Science}, 43:169--188, 1986.

\bibitem{Juditsky2020SparseRB}
Anatoli Juditsky, Andrei Kulunchakov, and Hlib Tsyntseus.
\newblock {Sparse recovery by reduced variance stochastic approximation}.
\newblock {\em Information and Inference: A Journal of the IMA},
  12(2):851--896, 11 2022.

\bibitem{juditsky2019unifying}
Anatoli Juditsky, Joon Kwon, and {\'E}ric Moulines.
\newblock Unifying mirror descent and dual averaging.
\newblock {\em arXiv preprint arXiv:1910.13742}, 2019.

\bibitem{juditsky2011first}
Anatoli Juditsky, Arkadi Nemirovski, et~al.
\newblock First order methods for nonsmooth convex large-scale optimization, i:
  general purpose methods.
\newblock {\em Optimization for Machine Learning}, 30(9):121--148, 2011.

\bibitem{juditsky2014deterministic}
Anatoli Juditsky and Yuri Nesterov.
\newblock Deterministic and stochastic primal-dual subgradient algorithms for
  uniformly convex minimization.
\newblock {\em Stochastic Systems}, 4(1):44--80, 2014.

\bibitem{koltchinskii2011nuclear}
Vladimir Koltchinskii, Karim Lounici, and Alexandre~B Tsybakov.
\newblock Nuclear-norm penalization and optimal rates for noisy low-rank matrix
  completion.
\newblock {\em The Annals of Statistics}, 39(5):2302--2329, 2011.

\bibitem{lecue2020robust}
Guillaume Lecu{\'e} and Matthieu Lerasle.
\newblock Robust machine learning by median-of-means: theory and practice.
\newblock {\em The Annals of Statistics}, 48(2):906--931, 2020.

\bibitem{lecue2020robust1}
Guillaume Lecu{\'e}, Matthieu Lerasle, and Timloth{\'e}e Mathieu.
\newblock Robust classification via mom minimization.
\newblock {\em Machine Learning}, 109(8):1635--1665, 2020.

\bibitem{lecue2018regularization}
Guillaume Lecu{\'e} and Shahar Mendelson.
\newblock Regularization and the small-ball method i: sparse recovery.
\newblock {\em The Annals of Statistics}, 46(2):611--641, 2018.

\bibitem{lei2020fast}
Zhixian Lei, Kyle Luh, Prayaag Venkat, and Fred Zhang.
\newblock A fast spectral algorithm for mean estimation with sub-gaussian
  rates.
\newblock In {\em Conference on Learning Theory}, pages 2598--2612. PMLR, 2020.

\bibitem{liu2019high}
Liu Liu, Tianyang Li, and Constantine Caramanis.
\newblock {High Dimensional Robust $M$-Estimation: Arbitrary Corruption and
  Heavy Tails}.
\newblock {\em arXiv preprint arXiv:1901.08237}, 2019.

\bibitem{liu2020high}
Liu Liu, Yanyao Shen, Tianyang Li, and Constantine Caramanis.
\newblock High dimensional robust sparse regression.
\newblock In {\em International Conference on Artificial Intelligence and
  Statistics}, pages 411--421. PMLR, 2020.

\bibitem{lounici2008sup}
Karim Lounici.
\newblock Sup-norm convergence rate and sign concentration property of lasso
  and dantzig estimators.
\newblock {\em Electronic Journal of statistics}, 2:90--102, 2008.

\bibitem{lounici2009taking}
Karim Lounici, Massimiliano Pontil, Alexandre~B Tsybakov, and Sara Van De~Geer.
\newblock Taking advantage of sparsity in multi-task learning.
\newblock {\em arXiv preprint arXiv:0903.1468}, 2009.

\bibitem{lugosi2019mean}
G{\'a}bor Lugosi and Shahar Mendelson.
\newblock Mean estimation and regression under heavy-tailed distributions: A
  survey.
\newblock {\em Foundations of Computational Mathematics}, 19(5):1145--1190,
  2019.

\bibitem{lugosi2019near}
G{\'a}bor Lugosi and Shahar Mendelson.
\newblock Near-optimal mean estimators with respect to general norms.
\newblock {\em Probability Theory and Related Fields}, 175(3):957--973, 2019.

\bibitem{lugosi2019regularization}
G{\'a}bor Lugosi and Shahar Mendelson.
\newblock Regularization, sparse recovery, and median-of-means tournaments.
\newblock {\em Bernoulli}, 25(3):2075--2106, 2019.

\bibitem{lugosi2019sub}
G{\'a}bor Lugosi and Shahar Mendelson.
\newblock Sub-gaussian estimators of the mean of a random vector.
\newblock {\em Annals of Statistics}, 47(2):783--794, 2019.

\bibitem{lugosi2021robust}
Gabor Lugosi and Shahar Mendelson.
\newblock Robust multivariate mean estimation: the optimality of trimmed mean.
\newblock {\em The Annals of Statistics}, 49(1):393--410, 2021.

\bibitem{minsker2018sub}
Stanislav Minsker.
\newblock Sub-gaussian estimators of the mean of a random matrix with
  heavy-tailed entries.
\newblock {\em The Annals of Statistics}, 46(6A):2871--2903, 2018.

\bibitem{minsker2015geometric}
Stanislav Minsker et~al.
\newblock Geometric median and robust estimation in banach spaces.
\newblock {\em Bernoulli}, 21(4):2308--2335, 2015.

\bibitem{minsker2022robust}
Stanislav Minsker, Mohamed Ndaoud, and Lang Wang.
\newblock Robust and tuning-free sparse linear regression via square-root
  slope.
\newblock {\em arXiv preprint arXiv:2210.16808}, 2022.

\bibitem{necoara2019linear}
Ion Necoara, Yu~Nesterov, and Francois Glineur.
\newblock Linear convergence of first order methods for non-strongly convex
  optimization.
\newblock {\em Mathematical Programming}, 175:69--107, 2019.

\bibitem{negahban2011estimation}
Sahand Negahban and Martin~J Wainwright.
\newblock Estimation of (near) low-rank matrices with noise and
  high-dimensional scaling.
\newblock {\em The Annals of Statistics}, 39(2):1069--1097, 2011.

\bibitem{negahban2012restricted}
Sahand Negahban and Martin~J Wainwright.
\newblock Restricted strong convexity and weighted matrix completion: Optimal
  bounds with noise.
\newblock {\em The Journal of Machine Learning Research}, 13(1):1665--1697,
  2012.

\bibitem{negahban2012unified}
Sahand~N Negahban, Pradeep Ravikumar, Martin~J Wainwright, and Bin Yu.
\newblock A unified framework for high-dimensional analysis of $ m $-estimators
  with decomposable regularizers.
\newblock {\em Statistical Science}, 27(4):538--557, 2012.

\bibitem{nemirovskij1983problem}
Arkadij~Semenovi{\v{c}} Nemirovskij and David~Borisovich Yudin.
\newblock {\em Problem Complexity and Method Efficiency in Optimization}.
\newblock Wiley-Interscience, 1983.

\bibitem{nesterov2015quasi}
Yu~Nesterov and Vladimir Shikhman.
\newblock Quasi-monotone subgradient methods for nonsmooth convex minimization.
\newblock {\em Journal of Optimization Theory and Applications},
  165(3):917--940, 2015.

\bibitem{nesterov2009primal}
Yurii Nesterov.
\newblock Primal-dual subgradient methods for convex problems.
\newblock {\em Mathematical Programming}, 120(1):221--259, 2009.

\bibitem{nesterov2013first}
Yurii Nesterov and Arkadi Nemirovski.
\newblock On first-order algorithms for l1/nuclear norm minimization.
\newblock {\em Acta Numerica}, 22:509--575, 2013.

\bibitem{scikit-learn}
F.~Pedregosa, G.~Varoquaux, A.~Gramfort, V.~Michel, B.~Thirion, O.~Grisel,
  M.~Blondel, P.~Prettenhofer, R.~Weiss, V.~Dubourg, J.~Vanderplas, A.~Passos,
  D.~Cournapeau, M.~Brucher, M.~Perrot, and E.~Duchesnay.
\newblock Scikit-learn: Machine learning in {P}ython.
\newblock {\em Journal of Machine Learning Research}, 12:2825--2830, 2011.

\bibitem{pensia2020robust}
Ankit Pensia, Varun Jog, and Po-Ling Loh.
\newblock Robust regression with covariate filtering: Heavy tails and
  adversarial contamination.
\newblock {\em arXiv preprint arXiv:2009.12976}, 2020.

\bibitem{HeavyTails}
Adarsh Prasad, Arun~Sai Suggala, Sivaraman Balakrishnan, and Pradeep Ravikumar.
\newblock Robust estimation via robust gradient estimation.
\newblock {\em Journal of the Royal Statistical Society: Series B (Statistical
  Methodology)}, 82(3):601--627, 2020.

\bibitem{raskutti2010restricted}
Garvesh Raskutti, Martin~J Wainwright, and Bin Yu.
\newblock Restricted eigenvalue properties for correlated gaussian designs.
\newblock {\em The Journal of Machine Learning Research}, 11:2241--2259, 2010.

\bibitem{rohde2011estimation}
Angelika Rohde and Alexandre~B Tsybakov.
\newblock Estimation of high-dimensional low-rank matrices.
\newblock {\em The Annals of Statistics}, 39(2):887--930, 2011.

\bibitem{sasai2022robust}
Takeyuki Sasai.
\newblock Robust and sparse estimation of linear regression coefficients with
  heavy-tailed noises and covariates.
\newblock {\em arXiv preprint arXiv:2206.07594}, 2022.

\bibitem{sasai2022outlier}
Takeyuki Sasai and Hironori Fujisawa.
\newblock Outlier robust and sparse estimation of linear regression
  coefficients.
\newblock {\em arXiv preprint arXiv:2208.11592}, 2022.

\bibitem{sedghi2014multi}
Hanie Sedghi, Anima Anandkumar, and Edmond Jonckheere.
\newblock Multi-step stochastic admm in high dimensions: Applications to sparse
  optimization and matrix decomposition.
\newblock {\em Advances in Neural Information Processing Systems}, 27, 2014.

\bibitem{shalev2011stochastic}
Shai Shalev-Shwartz and Ambuj Tewari.
\newblock {Stochastic Methods for $\ell_1$-Regularized Loss Minimization}.
\newblock {\em The Journal of Machine Learning Research}, 12:1865--1892, 2011.

\bibitem{shen2017tight}
Jie Shen and Ping Li.
\newblock A tight bound of hard thresholding.
\newblock {\em The Journal of Machine Learning Research}, 18(1):7650--7691,
  2017.

\bibitem{su2016slope}
Weijie Su and Emmanuel Cand{\`e}s.
\newblock Slope is adaptive to unknown sparsity and asymptotically minimax.
\newblock {\em The Annals of Statistics}, 44(3):1038--1068, 2016.

\bibitem{tibshirani1996regression}
Robert Tibshirani.
\newblock Regression shrinkage and selection via the lasso.
\newblock {\em Journal of the Royal Statistical Society: Series B
  (Methodological)}, 58(1):267--288, 1996.

\bibitem{tropp2015introduction}
Joel~A. Tropp.
\newblock An introduction to matrix concentration inequalities.
\newblock {\em Foundations and Trends® in Machine Learning}, 8(1-2):1--230,
  2015.

\bibitem{van2008high}
Sara~A Van~de Geer.
\newblock High-dimensional generalized linear models and the lasso.
\newblock {\em The Annals of Statistics}, 36(2):614--645, 2008.

\bibitem{van2009conditions}
Sara~A. van~de Geer and Peter B{\"u}hlmann.
\newblock {On the conditions used to prove oracle results for the Lasso}.
\newblock {\em Electronic Journal of Statistics}, 3(none):1360--1392, 2009.

\bibitem{vershynin2018high}
Roman Vershynin.
\newblock {\em High-Dimensional Probability: An Introduction with Applications
  in Data Science}, volume~47.
\newblock Cambridge university press, 2018.

\bibitem{watson1992characterization}
G~Alistair Watson.
\newblock Characterization of the subdifferential of some matrix norms.
\newblock {\em Linear Algebra and its Applications}, 170:33--45, 1992.

\bibitem{yin2018byzantine}
Dong Yin, Yudong Chen, Ramchandran Kannan, and Peter Bartlett.
\newblock Byzantine-robust distributed learning: Towards optimal statistical
  rates.
\newblock In {\em International Conference on Machine Learning}, pages
  5650--5659. PMLR, 2018.

\bibitem{yuan2006model}
Ming Yuan and Yi~Lin.
\newblock Model selection and estimation in regression with grouped variables.
\newblock {\em Journal of the Royal Statistical Society: Series B (Statistical
  Methodology)}, 68(1):49--67, 2006.

\bibitem{zhang2008sparsity}
Cun-Hui Zhang and Jian Huang.
\newblock The sparsity and bias of the lasso selection in high-dimensional
  linear regression.
\newblock {\em The Annals of Statistics}, 36(4):1567--1594, 2008.

\bibitem{zhang2009some}
Tong Zhang.
\newblock Some sharp performance bounds for least squares regression with l1
  regularization.
\newblock {\em The Annals of Statistics}, 37(5A):2109--2144, 2009.

\bibitem{zhao2006model}
Peng Zhao and Bin Yu.
\newblock On model selection consistency of lasso.
\newblock {\em The Journal of Machine Learning Research}, 7:2541--2563, 2006.

\bibitem{zou2006adaptive}
Hui Zou.
\newblock The adaptive lasso and its oracle properties.
\newblock {\em Journal of the American statistical association},
  101(476):1418--1429, 2006.

\end{thebibliography}

\end{document}